\documentclass{article}

\usepackage{microtype}
\usepackage{graphicx}
\usepackage{booktabs} % for professional tables
\usepackage{bm}
\usepackage{color, soul}
\usepackage{subfig}
\usepackage{amsmath}
\usepackage{amssymb}
\usepackage{mathtools}
\usepackage{amsthm}
\usepackage{mathmacros}
\usepackage{aux/macros}
\PassOptionsToPackage{numbers, square, sort}{natbib}
\usepackage[accepted]{aux/tmlr}
\usepackage[dvipsnames]{xcolor}
\usepackage[citecolor=NavyBlue, linkcolor=NavyBlue, colorlinks=True]{hyperref}
\usepackage{float}
\usepackage{enumitem}

\begin{document}

\title{Feature learning as alignment: a structural property of gradient descent in non-linear neural networks}
% Feature learning as alignment in general non-linear neural networks
% Gradient descent induces alignment between weights and the pre-activation tangents for general non-linear neural networks
% Alignment-based features are a structural property of gradient descent in neural networks

\author{\name Daniel Beaglehole$^{*,\dagger}$ \email dbeaglehole@ucsd.edu \\
      \addr UC San Diego
      \AND
      \name Ioannis Mitliagkas \email ioannism@google.com \\
      \addr Mila, Universit\'e de Montr\'eal \\
      \addr Google DeepMind
      \AND
      \name Atish Agarwala \email thetish@google.com\\
      \addr Google DeepMind}
      
\def\month{11}  % Insert correct month for camera-ready version
\def\year{2024} % Insert correct year for camera-ready version
\def\openreview{\url{https://openreview.net/forum?id=JXCe2ZcUXr}} % Insert correct link to OpenReview for camera-ready version

\maketitle
\def\thefootnote{*}\footnotetext{Corresponding author.}
\def\thefootnote{$\dagger$}\footnotetext{Work partially done as an intern at Google DeepMind.}

\begin{abstract}
    Understanding the mechanisms through which neural networks extract statistics from input-label pairs through feature learning is one of the most important unsolved problems in supervised learning. Prior works demonstrated that the gram matrices of the weights (the \emph{neural feature matrices}, NFM) and the  \textit{average gradient outer products} (AGOP) become correlated during training, in a statement known as the neural feature ansatz (NFA). Through the NFA, the authors introduce mapping with the AGOP as a general mechanism for neural feature learning. However, these works do not provide a theoretical explanation for this correlation or its origins. In this work, we further clarify the nature of this correlation, and explain its emergence. We show that this correlation is equivalent to alignment between the left singular structure of the weight matrices and the newly defined \emph{pre-activation tangent features} at each layer. We further establish that the alignment is driven by the interaction of weight changes induced by SGD with the pre-activation features, and analyze the resulting dynamics analytically at early times in terms of simple statistics of the inputs and labels. We prove the derivative alignment occurs almost surely in specific high dimensional settings. Finally, we introduce a simple optimization rule motivated by our analysis of the centered correlation which dramatically increases the NFA correlations at any given layer and improves the quality of features learned.
\end{abstract}

\section{Introduction}

Neural networks have emerged as the state-of-the-art machine learning methods for seemingly complex tasks, such as language generation \citep{GPT}, image classification \citep{AlexNet}, and visual rendering \citep{NeRF}. The precise reasons why neural networks generalize well have been the subject of intensive exploration, beginning with the observation that standard generalization bounds from statistical learning theory fall short of explaining their performance~\citep{zhang2021understanding}. 

A promising line of work emerged in the form of the neural tangent kernel, connecting neural networks to kernels in the wide limit~\citep{jacot2018neural,ChizatLazyTraining}.  
However, subsequent research showed that the success of neural networks relies critically on aspects of learning which are absent in kernel approximations~\citep{ghorbani2019limitations,allen-zhu2019what,yehudai2019power,li2020learning,refinetti2021classifying}.
Other work showed that low width suffices for gradient descent to achieve arbitrarily small test error \citep{ji2020knowledge}, further refuting the idea that extremely wide networks are necessary. 

Subsequently, the success of neural networks has been largely attributed to
\emph{feature learning} - the ability of neural networks to learn representations of data which are useful for
downstream tasks.
However, the specific mechanism through which features are learned is an important unsolved problem in deep learning theory. A number of works have studied the abilities of neural networks to learn features in structured settings~\citep{abbe2022merged,ba2022high-dimensional,EshaanFeatureLearning,barak2022hidden,damian2022neural,moniri2023theory, WillettAGOP}. Some of that work proves strict separation in terms of sample complexity between neural networks trained with stochastic gradient descent and kernels~\citep{mousavi2022neural}.

The work above studies simple structure, such as
learning from low-rank data or functions that are hierarchical compositions of simple elements.
Recent work makes a big step towards generalizing these assumptions by proposing the \emph{neural feature ansatz} (NFA)~\citep{RadhakrishnanNFA,BeagleholeConvNFA}, 
a general structure that emerges in the weights of trained neural networks. The NFA states that the gram matrix of the weights at a given layer (known as the \textit{neural feature matrix} or NFM) is aligned with the \textit{average gradient outer product} (AGOP) of the network with respect to the input to that layer.
In particular, the NFM and AGOP are highly correlated in all layers of trained neural networks of general architectures, including practical models such as VGG \citep{VGG}, vision transformers \citep{ViT}, and GPT-family models \citep{GPT}. 

A major missing element is an explanation for \emph{how and why} the AGOP and NFM become correlated through training with gradient descent.
In this paper, we precisely describe the emergence of this correlation. We establish that the NFA is equivalent to alignment between the left
singular structure of the weight matrices and the uncentered covariance of the \textit{pre-activation tangent kernel} (PTK) (Section~\ref{sec: alignment}) features. We then introduce the
\emph{centered neural feature correlation} (\CNFC{}) which isolates this alignment process.
We show empirically that the \CNFC{} is close to its maximum value of $1$ at early times, and fully captures the NFA at late times
for a variety of architectures (fully-connected, convolutional, and attention layers) over a diverse collection of datasets (Section~\ref{sec: centering}).
Our experiments suggest that the \CNFC{} drives the development of the NFA. Through this centering, we show that the dynamics of the
\CNFC{} can be understood analytically in terms of the statistics of the data and 
labels at early times (Section~\ref{sec: cnfc theory}). Using this decomposition, we
show that the NFA emerges as a structural property of gradient descent (analytical result in the commonly studied setting of uniform data on the sphere, under certain assumptions on the activation and target functions). In particular, the first non-zero derivatives of the centered NFM and AGOP will be asymptotically identical. We further characterize how the
NFA depends on the data distribution, and explore this analytically and experimentally.

Finally, motivated by our theory, we design an intervention to increase the influence of the \CNFC{} and make the NFA more robust:
\OptName{}, a layerwise gradient normalization scheme (Section~\ref{sec: learning speeds}). The effectiveness of the latter update rule suggests a path towards rational design of architectures and training procedures that maximize the NFA notion of feature learning by promoting alignment dynamics.

\section{Alignment between the weight matrices and the pre-activation tangents} 
\label{sec: alignment}

In this section, we decompose the AGOP into the weight matrices and the feature covariance of the pre-activation tangent kernel (PTK), and demonstrate that the NFA is equivalent to alignment between the weights and these PTK features. 
We include a glossary of terms for ease of reference in Appendix~\ref{app: glossary}.

\subsection{Preliminaries}

We consider fully-connected neural networks with a single output of depth $L \geq 1$, where $L$ is the number of hidden layers, written $f : \Real^d \goto \Real$. We write the input to layer $\ell \in \{0, \ldots, L\} $ as $x_\ell$, where $x_0 \equiv x$ is the original datapoint, and the pre-activation as $h_\ell(x)$. Then, 
\begin{align}
    h_\ell(x) = W^{(\ell)} x_\ell,\qquad x_{\ell+1} = \phi(h_\ell(x))~,\label{eq: neural net setup}
\end{align} 
where $\phi$ is an element-wise nonlinearity, $W^{(\ell)} \in \mathbb{R}^{k_{\ell+1} \times k_{\ell}}$ is a weight matrix, and $k_\ell$ is the hidden dimension at layer $\ell$. We restrict $k_{L+1}$ to be the number of output logits, and set $k_0 = d$, where $d$ is the input dimension of the data. Note that $f(x) = h_{L}(x)$ and we ignore $x_{L+1}$. We train $f$ by gradient descent on a loss function
$\Lo(\theta, X)$, where $X$ is an input dataset, and $\theta$ is the collection of weights.

We consider a supervised learning setup where we are provided $n$ input-label pairs $(x^{(1)}, y^{(1)}), \ldots, (x^{(n)}, y^{(n)}) \in \Real^d \times \Real$. We denote the network inputs (datapoints) $X \in \Real^{n \times d}$ and the labels $y \in \Real^{n \times 1}$. For a given network, the inputs to intermediate layers $\ell \in \{0,\ldots,L\}$ are written $X_\ell \in \Real^{n \times k_\ell}$, where $X_0 \equiv X$. We train a fully-connected neural network to learn the mapping from network inputs to labels by minimizing a 
standard loss function, such as mean-squared-error or cross-entropy, on the dataset. 

\begin{figure*}[h]
    \centering
    \includegraphics[scale=0.5]{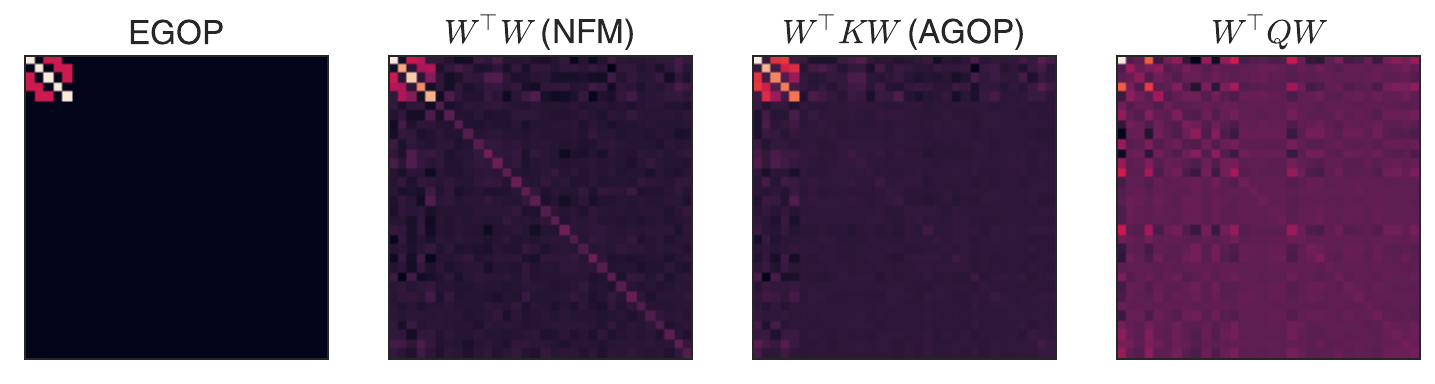}
    \caption{Various feature learning measures for target function
    $y(x) = \sum_{k=1}^{r} x_{k \mod r} \cdot x_{(k+1) \mod r}$ with
    $r=5$ and inputs drawn from standard normal. The EGOP $\Exp{x \sim \mu}{\dfdx{y}{x} \dfdx{y}{x}\tran}$ (first plot)
    captures the low-rank structure of the task. The NFM
    $\round{W\tran W}$ (second plot) and AGOP $\round{W\tran K W}$
    (third plot) of a fully-connected network are similar to each other
    and the EGOP. Replacing $K$ with a symmetric matrix $Q$ with
    the same spectrum but independent eigenvectors obscures the
    low rank structure (fourth plot), and reduces the correlation from
    $\corr{\NFM, \AGOP} = 0.93$ to $\corr{\NFM, W\tran Q W} = 0.53$.}
  \label{fig: Corrupted Ansatz}
\end{figure*}

One can define two objects associated with neural networks that capture learned structure. For a given layer $\ell$, the \textit{neural feature matrix} (NFM) $\NFM_\ell$ is the gram matrix of the columns of the weight matrix $W^{(\ell)}$, i.e. $\NFM_\ell \equiv (W^{(\ell)})\tran W^{(\ell)}$. $\NFM_\ell$ depends on the right singular vectors (and corresponding singular values)
of $W^{(\ell)}$. The second fundamental object we consider is the \textit{average gradient outer product} (AGOP) $\AGOP_\ell$, defined as $\AGOP_\ell \equiv \frac{1}{n} \sum_{\alpha=1}^n \dfdx{f(x^{(\alpha)})}{x_\ell} \dfdx{f(x^{(\alpha)})}{x_\ell}\tran$.

Since both these matrices have the same dimensions ($k_{\ell+1}\times k_{\ell+1}$),
we can consider their cosine similarity. We define the 
\textit{neural feature correlation}  (NFC) by:
\begin{align*}
    \corr{\NFM_{\ell}, \AGOP_{\ell}} \equiv \trace{\NFM_{\ell}\tran \AGOP_{\ell}} \cdot \trace{\NFM_{\ell}\tran \NFM_{\ell}}^{-1/2} \cdot \trace{\AGOP_{\ell}\tran \AGOP_{\ell}}^{-1/2}~. \tag{NFC}
\end{align*}
This takes on values in $[0,1]$ since $\NFM_{\ell}$ and $\AGOP_{\ell}$ are both
PSD.

Prior work has shown that in trained neural networks, the NFC will generally be close to $1$ to varying degree of approximation. This notion is formalized in the \textit{Neural Feature Ansatz} (NFA):

\begin{ansatz}[Neural Feature Ansatz \citep{RadhakrishnanNFA}]
\label{ansatz:nfa}
The Neural Feature Ansatz states that, for all layers $\ell \in [L]$ of a
fully-connected
neural network with $L$ hidden layers trained on input data $x^{(1)}, \ldots, x^{(n)}$, the NFC will satisfy $\corr{\AGOP_{\ell}, \NFM_{\ell}} \approx 1~.$
\end{ansatz}

The inputs to the covariance are the NFM and AGOP with respect to the input of layer $\ell$. Here, $\dfdx{f(x)}{x_\ell} \in \Real^{k_{\ell} \times 1}$ denotes the gradient of the function $f$ with respect to the intermediate representation $x_\ell$. For simplicity, we may concatenate these gradients  across the input dataset into a single matrix $\dfdx{f(X)}{x_\ell} \in \Real^{n \times k_{\ell}}$. Note we consider scalar outputs in this work, though the NFA relation is identical when there are $c \geq 1$ outputs, where in this case $\dfdx{f(x)}{x_\ell} \in \Real^{k_{\ell} \times c}$ is the input-output Jacobian of the model $f$.

% \aga{The introduction of the initialization scale reads a bit awkwardly here. Do we use the same definition later? If we want to keep this, describe the experiment in more detail (MLP initialized with etc etc)}
We note that while the ansatz states exact proportionality, in practice, the NFM and AGOP are highly correlated with correlation less than $1$, where the final correlation depends on many aspects of training and architecture choice (discussed further in Section~\ref{sec: centering}).    

The relation between the NFM and the AGOP is significant, in part, because for a neural network that has learned enough information 
about the target function, the AGOP of this model with respect to the first-layer inputs will approximate the \emph{expected gradient outer product}
(EGOP) of the target function \citep{SmoothedEGOP}. In 
particular, the EGOP of the target function contains task-specific structure that is 
completely independent of the model used to estimate it. Where the 
labels are generated from a particular target function $y(x) : \Real^d \goto \Real$ on data sampled from a distribution $\mu$, the EGOP is defined as
\begin{align*}
    \EGOP(y,\mu) = \Exp{x \sim \mu}{\dfdx{y}{x} \dfdx{y}{x}\tran}~.
\end{align*}
If the neural feature ansatz (Ansatz~\ref{ansatz:nfa}) holds, the correlation of the EGOP and the AGOP at the end of training also implies high correlation between the NFM of the first layer and the EGOP, so that the NFM has encoded this task-specific structure.

To demonstrate the significance of high correlation between the NFM and the AGOP in successfully trained networks, we consider the following \textit{chain-monomial} low-rank task:
\begin{align}
    y(x) = \sum_{k=1}^{r} x_{k \mod r} \cdot x_{(k+1) \mod r}~,
    \label{eq:chain monomial}
\end{align}
where the data inputs are sampled from an isotropic Gaussian distribution $\mu = \mathcal{N}(0,I)$. In this case, the entries $\EGOP(y, \mu)$ will be $0$ for rows and columns outside of the $r \times r$ sub-matrix corresponding to $x_1, \ldots, x_r$ (Figure~\ref{fig: Corrupted Ansatz}), as $y$ does not vary with coordinates $x_{r+1},\ldots,x_d$. Within this sub-matrix, the diagonal entries will have value $2$, while the off-diagonal entries will be either $1$ or $0$. Therefore, $\EGOP(y,\mu)$ will be rank $r$, where $r$ is much less than the ambient dimension. 

We verify for this task that the AGOP of the trained model resembles the EGOP (first and third panels of Figure~\ref{fig: Corrupted Ansatz}). Here the NFA holds and the NFM (second panel)
resembles the AGOP and therefore the EGOP as well.
Therefore, the neural network has learned the model-independent and task-specific structure of the chain-monomial 
task in the right singular values and vectors of the first layer weight matrix, as these are determined by the NFM.
In fact, previous works have demonstrated that the NFM of the first layer of a well-trained neural network
is highly correlated with the AGOP of a fixed kernel method trained on the same dataset \citep{RadhakrishnanNFA}.

This insight has inspired iterative kernel methods which can
match the performance of fully-connected networks \citep{RadhakrishnanNFA, LinearRFM, FCM} and improve over fixed convolutional kernels \citep{BeagleholeConvNFA}. Additional prior works demonstrate the benefit of including the AGOP features to improve feature-less predictors \citep{SpokoinyEGOP, ConsistentEGOP, GradientWeights}. Additionally, because the NFM is correlated with the AGOP, the AGOP can be used to recover the features from feature-less methods, such as kernel machines. The AGOP has also been show to capture surprising phenomena of neural networks beyond low-rank feature learning including deep neural collapse \citep{BeagleholeNC}. 

\subsection{Alignment decomposition}

In order to understand Ansatz \ref{ansatz:nfa}, it is useful to decompose the AGOP. Doing so will allow us to show that the neural feature correlation (NFC) can be interpreted
as an alignment metric between weight matrices and the \textit{pre-activation tangent kernel} (PTK). The PTK $\Kn^{(\ell)}(x, z)$ is defined with respect to a layer $\ell$ of a neural network and two inputs $x,z$. The kernel evaluates to:
\begin{align*}
\Kn^{(\ell)}(x, z) \equiv & \dfdx{f(x)}{h_\ell}\cdot\dfdx{f(z)}{h_\ell}~. \tag{PTK}
\end{align*}
When the arguments $x$ and $z$ are omitted, $\Kn^{(\ell)} \in \Real^{n \times n}$ consists of the matrix of kernel evaluations on all pairs of datapoints. We may also consider the covariance $K^{(\ell)} \in \Real^{k_{\ell} \times k_{\ell}}$ of the features associated with the PTK, $\dfdx{f(X)}{h_\ell} \in \Real^{n \times k_{\ell}}$. In particular, we define the \emph{feature covariance} of the PTK as,
\begin{align*}
   \phantom{PTK feature covariance} K^{(\ell)} \equiv & \dfdx{f(X)}{h_\ell}\tran \dfdx{f(X)}{h_\ell}~. \tag{PTK feature covariance}
\end{align*}
Crucially, for any layer $\ell$, we can re-write the AGOP in terms of this feature covariance as,
\begin{align*}
\AGOP_\ell = (W^{(\ell)})\tran K^{(\ell)} W^{(\ell)}~.
\end{align*}
This gives us the following proposition:
\begin{proposition}[Alignment decomposition of NFC]
\label{prop: ptk, nfa alignment}
\begin{align*}
    \corr{\NFM_\ell, \AGOP_\ell} = \corr{(W^{(\ell)})\tran W^{(\ell)}, (W^{(\ell)})\tran K^{(\ell)} W^{(\ell)}}~.
\end{align*}
\end{proposition}
This alignment holds trivially and exactly if $K^{(\ell)}$ is the
identity. However, the correlation can be high in trained networks where $K^{(\ell)}$ is non-trivial. For example, in the
chain monomial task (Figure~\ref{fig: Corrupted Ansatz}),
$K^{(0)}$ is far from identity (standard deviation of its eigenvalues is $5.9$ times its average eigenvalue),
but the NFA correlation is $0.93$ at the end of training. We also note that if $K^{(\ell)}$ is independent of $W^{(\ell)}$,
the alignment is lower than in trained networks; in the same
example, replacing $K^{(0)}$ with a matrix $Q$ with equal spectrum but random eigenvectors greatly reduces the correlation to $0.53$
and qualitatively disrupts the structure relative to the NFM
(Figure~\ref{fig: Corrupted Ansatz}, rightmost column). We show the same result for the CelebA dataset (see Appendix~\ref{app: real datasets}).  
Therefore, the NFA is a consequence of alignment between the left
eigenvectors of $W^{(\ell)}$ and $K^{(\ell)}$ in addition to spectral considerations.

Note that $K^{(\ell)}$ itself is the AGOP of the neural network $f$ with respect to the pre-activations at layer $\ell$, i.e.
\begin{equation}
K^{(\ell)} \equiv \frac{1}{n} \sum_{\alpha=1}^n \dfdx{f(x^{(\alpha)}_\ell)}{h_\ell} \dfdx{f(x^{(\alpha)}_\ell)}{h_\ell}\tran.
\end{equation}
For convolutional and attention layers, the $K^{(\ell)}$ are computed by additionally averaged over all patches and token positions in the input, respectively (see \citet{RadhakrishnanNFA, BeagleholeConvNFA} for the formulation of the NFA in these architectures).

\begin{figure*}
    \centering
    \subfloat{\includegraphics[scale=0.575]{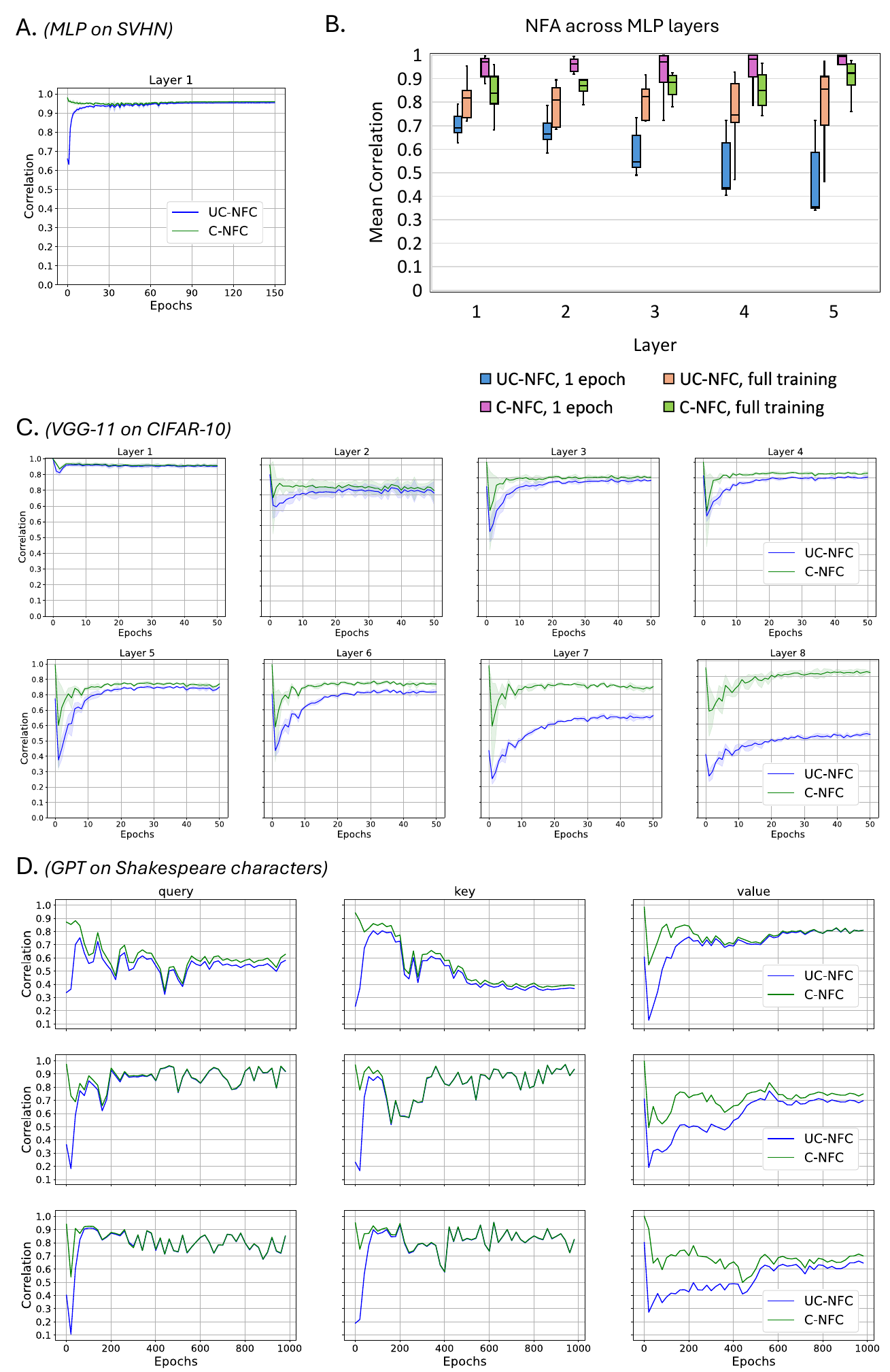}}
    \caption{Uncentered and centered neural feature correlations across (A,B) fully-connected, (C) convolutional, and (D) attention layers with large initialization scale. (A,C,D) show trajectories of C/UC-NFC over training. (B) shows NFC values across all layers of an MLP with five hidden layers, averaged over CIFAR-10, CIFAR-100, SVHN, MNIST, GTSRB, and STL-10 datasets. (A-C) are additionally averaged over three random seeds. Each row of (D) is an attention block (ordered from first to last in the GPT model), while the columns show correlations for query, key, and value layers, respectively.}
    \label{fig: real NFA measurements}
\end{figure*}

\section{Centering the NFC isolates weight-PTK alignment}
\label{sec: centering}

We showed that the neural feature ansatz is equivalent to PTK-weight alignment (Proposition~\ref{prop: ptk, nfa alignment}).
We now ask:
is the increase in the NFC due to alignment of the weight
matrices to the current PTK, or the alignment of the PTK to
the current weights? In practice, both effects matter, but numerical evidence suggests that changes in the PTK do not drive the early dynamics of the NFC (Appendix~\ref{app: other centerings}). Instead, we observe that the alignment 
between the weights and the PTK feature covariance is driven by a
\emph{centered} NFC
that captures alignment between the \emph{parameter changes} and the PTK. We then show this centered NFC can hold robustly in settings where the NFC holds with correlation less than 1, such as early in training and/or with large initialization. Finally, we establish analytically through the centered NFC how the NFA emerges as a structural property of gradient descent.  

We begin by considering a decomposition of the NFM and AGOP into parts that depend on initialization
and a part that depends on the changes in the weight matrix.
Let $W^{(\ell)}_{t}$ and $K^{(\ell)}_{t}$ be the
weight matrix and PTK feature covariance for layer $\ell$ at time $t$. We can write the NFM and AGOP in terms of the initial weights $W^{(\ell)}_{0}$ and
the change in weights $\Wb^{(\ell)}_{t} \equiv W^{(\ell)}_{t}-W^{(\ell)}_0$
as follows:
\begin{equation}
\begin{split}
    \NFM{}  = W_{t}\tran W_{t} & = \Wb_{t}\tran \Wb_{t} + W_0\tran \Wb_{t} + \Wb_{t}\tran W_0 + W_0\tran W_0,\\
    \AGOP{}  = W_{t}\tran K_{t} W_{t} & = \Wb_{t}\tran K_{t} \Wb_{t} + W_0\tran K_{t} \Wb_{t} + \Wb_{t}\tran K_{t} W_0 + W_0\tran K_{t} W_0
\end{split}
\label{eq:nfa_decomp}
\end{equation}
where we omitted $\ell$ from all terms for ease of notation. The first term in each decomposition
isolates the gram matrix of the \emph{changes} in the weight matrix. We call
$\Wb\tran \Wb$ the \emph{centered NFM} and $\Wb\tran K \Wb$ the \emph{centered AGOP}.
The centered AGOP in particular measures the
alignment of weight updates with the current PTK feature covariance.
Both terms are $0$ at initialization, and if the weight matrices change significantly
(that is, if $||\Wb||\gg ||W_{0}||$), both the NFM and AGOP are dominated by
the centered terms. (Note: in the limit that initialization scale vanishes to $0$ (i.e. small initialization), then we expect the centered and uncentered NFM and AGOP to coincide, as $\Wb = W_t - W_0 \rightarrow W_t$, where $t$ is a step at the end of training.)

The decomposition in Equation \ref{eq:nfa_decomp} suggests that updates which
drive correlation of the centered NFM and centered AGOP can also drive the overall value of the
original NFC. Inspired by this observation, we define the \emph{centered NFC}
(\CNFC{}) as, $$ \corr{(\Wb^{(\ell)})\tran \Wb^{(\ell)},(\Wb^{(\ell)})\tran K^{(\ell)} \Wb^{(\ell)}}~.$$ For the remainder of the paper, we will refer to the
original NFC as the \textit{uncentered NFC} (\UCNFC{}) to avoid ambiguity.

The centered NFC is consistently higher than the uncentered NFC across training times, architectures,
and datasets (Figure~\ref{fig: real NFA measurements}). This especially holds at early times and in deeper layers of a network
(Figure~\ref{fig: real NFA measurements}C and D, VGG-11 on CIFAR-10 and
the GPT model on Shakespeare respectively). For MLPs, we conduct a broader set of experiments and verify that the trends
hold on a wide range of vision datasets. We also note that the \CNFC{} at early times is relatively robust to the initialization statistics of the weight matrix $W^{(\ell)}$ - unlike the
\UCNFC{} (for additional experiments, see Appendix~\ref{app: nfa, vary data spectrum}). We will return to this point in Section~\ref{sec: learning speeds}. These experiments suggest the \CNFC{} is responsible for improvements in the uncentered correlation.

High \CNFC{} values can drive increases in the
\UCNFC{} as long as the centered NFM and centered AGOP are increasing in magnitude
during training. We can confirm that the weights move significantly from initialization,
which drives the contribution from the centered NFM and centered AGOP (Figures~\ref{fig: losses, MLP}, and \ref{fig: losses, VGG} in Appendix~\ref{app: real datasets}).
The increased importance of the \CNFC{} leads the \UCNFC{} to converge to the \CNFC{} at late times in most of our experimental settings, as the contribution from the weight changes
dominates the contribution from the initialization.
These findings validate that the \CNFC{} is an important contributor to increases in the \UCNFC{}.

Our experiments suggest that studying the \CNFC{} is a useful first step to understanding the neural feature ansatz.
In Section~\ref{sec: cnfc theory}, we develop a theoretical analysis to
understand why the \CNFC{} is generically large at early times. We then use this
analysis to predict the value of the \CNFC{} and motivate interventions which can keep the \CNFC{} high and promote
the NFA earlier on in training (Section~\ref{sec: learning speeds}).

\section{Theoretical analysis of the \CNFC{} at early times}\label{sec: cnfc theory}

\subsection{Gradient flow dynamics}

We now theoretically identify why the centered AGOP and NFM become correlated as a structural property of gradient descent, for at least the early training times.
We consider the setup introduced in Equation~\eqref{eq: neural net setup}. For theoretical convenience we focus on the case of training under gradient flow,
where the dynamics of a weight matrix $W^{(\ell)}$ trained on loss $\Lo$ is given by
\begin{equation}
\dot{W}^{(\ell)} \equiv \frac{dW^{(\ell)}}{dt} =  -\nabla_{W^{(\ell)}}\Lo~,
\end{equation}
where $\dW^{(\ell)}$ is the time derivative of $W^{(\ell)}$. We expect our results to apply to gradient descent with small and intermediate learning rates as well. For simplicity of notation, we omit the layer index $\ell$.

We note that $\Wb = 0$. This gives us the following proposition, which shows that early time dynamics of the centered NFM and AGOP are dominated by the \emph{second} time derivatives:
\begin{proposition}[Centered NFC dynamics]
\label{prop: Early NFA dynamics} Let $W$ be the weights of a fully-connected layer of a neural network at initialization and $X$ be the inputs to that layer. Then, when the neural network is trained by gradient flow on a loss function $\Lo$, we have $\bar{W}\tran \bar{W} = \bar{W}\tran K \bar{W} = 0$, $\frac{\dif}{\dif t} (\bar{W}\tran \bar{W}) = \frac{\dif}{\dif t} (\bar{W}\tran K \bar{W}) = 0$, and the first non-zero time derivatives satisfy,
\begin{align}
\frac{\dif^2}{\dif t^2} (\bar{W}\tran \bar{W}) = 2 \cdot X\tran \Ldot \Kn \Ldot X,\quad \frac{\dif^2}{\dif t^2} (\bar{W}\tran K \bar{W}) = 2 \cdot X\tran \Ldot \Kn^2 \Ldot X~
\end{align}
where $\Ldot$ is the $n\times n$ diagonal matrix of logit derivatives $\Ldot \equiv {\diag{\frac{\partial \Lo}{\partial f}} }$.
\end{proposition}
\begin{proof}[Proof of Proposition~\ref{prop: Early NFA dynamics}]
At initialization, all terms containing at least one copy of $\bar{W}$ vanish, leading to the zero and first time derivatives to be 0. Further, we have $\dW = \dfdx{f(X)}{h}\tran \Ldot X$. Therefore, using that $K \equiv \dfdx{f(X)}{h}\tran \dfdx{f(X)}{h}$ and $\Kn \equiv \dfdx{f(X)}{h} \dfdx{f(X)}{h}\tran$,
\begin{align*}
    \frac{1}{2} \frac{\dif^2}{\dif t^2} (\bar{W}\tran \bar{W}) = \dW\tran \dW = X\tran \Ldot \dfdx{f(X)}{h} \dfdx{f(X)}{h}\tran \Ldot X = X\tran \Ldot \Kn \Ldot X~,
\end{align*}
and,
\begin{align*}
    \frac{1}{2} \frac{\dif^2}{\dif t^2} (\bar{W}\tran \bar{W}) = \dW\tran K \dW &= X\tran \Ldot \dfdx{f(X)}{h}\, K\, \dfdx{f(X)}{h}\tran \Ldot X\\
    &= X\tran \Ldot \dfdx{f(X)}{h} \dfdx{f(X)}{h}\tran \dfdx{f(X)}{h} \dfdx{f(X)}{h}\tran \Ldot X\\
    &= X\tran \Ldot \Kn^2 \Ldot X~.
\end{align*}
\end{proof}

We immediately see how gradient-based training can drive
the \CNFC{} towards $1$ at early times; the first non-trivial
time derivatives are often highly correlated as they differ by only a single matrix power of $\Kn$, and have the same dependence on the labels.
If $\Kn$ is proportional to a projection matrix, then the two derivatives will have perfect correlation;
even if not, $\Kn$ and $\Kn^{2}$ will have identical eigenvectors, and hence we might expect a range of spectra will enable high correlation between them. In the case that $\Kn$ and $\Kn^2$ are sufficiently
correlated, the \CNFC{} is driven to a 
high value immediately upon training, and the high value of the \CNFC{} eventually drives the \UCNFC{} and causes the NFA to hold.

In the remainder of this section, we study this correlation in two different high-dimensional limits. Our analysis
suggests that for large models the \CNFC{} has a generic tendency to increase early
in training.

\subsection{Maximum \CNFC{} for data on the sphere}

We first provide a general and well-studied setting where the first non-zero derivatives of the centered NFM and AGOP are perfectly correlated - uniform data on the sphere in high dimensions \citep{ghorbani2020when, ghorbani2021linearized, theo2022}. This limit
corresponds to an infinite width network which, combined with the data symmetry,
induces the PTK matrix $\Kn$ to act like a projection matrix.

\paragraph{Data setup} We sample data $x \sim \sqrt{d} \cdot \mathrm{Unif}\round{\S^{d-1}}$ uniformly distributed on the sphere in $d$ dimensions with radius $\sqrt{d}$. We assume the labels are generated from a target function that maps $f^* : \S^{d-1} \rightarrow [-d, d]$. I.e. the label for the point $x$ is equal to $f^*(x)$. We train the parameters $a,W$ for a one-hidden layer fully-connected neural network $f(x) = f(x; a, W) = a\tran \phi(Wx)$ with element-wise activation function $\phi$.
For a learning trajectory $(a_{t}, W_{t})$, the initial values $a_{0}$ and $W_{0}$
are sampled
i.i.d. standard Gaussian with variance at most $O(1/k)$. I.e. for all $i,j,\ell$, we have $a_0(i), W_0(j, \ell) \sim \mathcal{N}(0,c)$ for $c = O(1/k)$.

\paragraph{Notation} We will use asymptotic notation $O_d, o_d, \Theta_d, \Omega_d, \omega_d$ in the usual way, where the limits are taken with respect to the data dimension $d$. $\wt{O}_d, \wt{o}_d, \wt{\Theta}_d, \wt{\Omega}_d, \wt{\omega}_d$ are equivalent to their previous definitions but hide dependencies of polylogarithmic functions of $d$. We will use $\|\cdot\|$ to refer to $\ell_2$ vector norm and the operator norm of a matrix, i.e. $\|A\| = \max_{v \in \Real^d} \frac{\|Av\|}{\|v\|}$, for a square matrix $A \in \Real^{d \times d}$. Recall we sample $n$ datapoints and use hidden dimension $k$ so that $a_t \in \Real^{k \times 1}$ and $W_t \in \Real^{k \times d}$. We write the vector of all ones in $p$ dimensions as $\one \in \Real^{p \times 1}$. We will omit the time index $t$ for ease of notation.

We make the assumption that $f^*$ has is bounded by a polylogarithmic factor of dimension
with superlinearly vanishing probability:
\begin{assumption}[Labels bounded]
\label{assump: bounded labels}
We have, for any positive $\delta>0$, with probability $1 - O_{d}(d^{-1-\delta})$, $|f^*(x)| = O_d(\log^\ell(d))$ for a constant integer $\ell$.
\end{assumption}
We also assume that the target function has a non-vanishing linear component:
\begin{assumption}[Non-trivial linear component]
\label{assump: linear component}
We have, $\|\Exp{x}{f^*(x) x}\| = \Omega_d(1)$.
\end{assumption}

These assumptions hold for example when $f^*$ is exactly linear, i.e. $f^*(x) = \beta\tran x$ for $\|\beta\| \leq 1$. In this case, the inner product $x\tran \beta = \wt{O}_d(1)$ with probability $O_d(d^{-1-\delta})$ for $x$ uniform on $\sqrt{d}\cdot\S^{d-1}$ and any positive $\delta$.

We make an additional assumption on the differentiability of the PTK, inherited from the activation $\phi$, according to the conditions from \citet{theo2022}. This assumption is needed for our analysis, in which differentiability enables us to Taylor expand the PTK matrix \citep{ElKarouiLinearize}. We restate this assumption.
\begin{assumption}[Differentiability of the PTK]
\label{assump: diff activation}
Let $\Kn(x,x') = \Exp{a_k, w_k}{a_k^2 \phi'(w_k\tran x) \phi'(w_k\tran x')} = h_d(\inner{x,x'}/d)$ be the PTK for the network $f$, where $h_d : [-1,1] \rightarrow \Real$ is a positive semi-definite kernel function, defined for each dimension $d$. We assume there exist finite $h(0), h'(0), h''(0) > 0$ such that $\lim_{d\goto\infty} h_d(0) = h(0)$, $\lim_{d\goto\infty} h'_d(0) = h'(0)$, and $\lim_{d\goto\infty} h''_d(0) = h''(0)$, where the first and second derivatives of $h_d$, $h'_d$ and $h''_d$, are assumed to exist on $[-1,1]$ for all $d$. 
\end{assumption}
Note this is a sub-case of Assumption~1 in \citet{theo2022} for level $\ell = 1$, and is satisfied for $h_d$ that is twice differentiable everywhere. We now introduce our final simplifying assumption on the activation function and data:
% \red{\begin{assumption}[PTK mean zero.]
% \label{assump: ptk mean zero}
% We assume $h_d(0) = 0$.
% \end{assumption}
% This assumption has been used before directly on the conjugate kernel in prior work (e.g. in  \citet{AdlamLinearize2}), and is satisfied for e.g. the quadratic activation $\phi(z) = z^2$. Here, we apply this assumption to guarantee that the lowest order term in the Taylor expansion of the PTK matrix is the linear term $XX\tran$ and not the mean term $\one\one\tran$. Otherwise, the mean term and the linear term coarsely could be of the same order, which significantly complicates the analysis. In this case, we would need to carefully control the asymptotics of these two quantities and the products with $X, Y$ appearing in Proposition~\ref{prop: Early NFA dynamics}.}

We will consider $f$ trained using mean-squared error (MSE) loss. With this loss, $\Ldot$ corresponds to the diagonal matrix of the residuals $y-f(x)$.
If the outputs of the network
is $0$ on the training data, then 
$\Ldot = Y \equiv \diag{y}$, the labels themselves. We can guarantee this by either of the following methods:
\begin{method}[Subtract copy]
\label{method 1}
Subtracting off an identical (untrained) copy of the neural network from each output at initialization. I.e. we train the parameters of the network $f$ on the loss with respect to outputs $\hat{f}_{t}(x) = f_{t}(x) - f_0(x)$, where $f_0(x)$ is a copy of $f$ at initialization and $f_{t}(x) = f(x;a_{t}, W_{t})$
\end{method}
\begin{method}[Small initialization]
\label{method 2}
Initializing $W=0$ or $\|a\| = \epsilon$ for $\epsilon>0$ arbitrarily small. 
\end{method}
We now state our main theorem.
\begin{theorem}[Maximum \CNFC{}]
\label{thm: high c-nfc}
Suppose the data $X$ are sampled uniformly at random from $\S^{d-1}$ and the labels are generated by $f^*$ satisfying Assumptions~\ref{assump: bounded labels} and \ref{assump: linear component}. Suppose we train $f$ with MSE loss and initialize the network so that the initial outputs are $0$ by either Method~\ref{method 1} or Method~\ref{method 2} above. Assume the activation function $\phi$ satisfies Assumption~\ref{assump: diff activation}. We consider the regime that $\omega_d\round{d\log^{2\ell+1}d} \leq n \leq o_d\round{d^{2-\delta}}$ for some $\delta > 0$, and width $k \rightarrow \infty$. Then, at $t=0$ in training (in the setting of Proposition~\ref{prop: Early NFA dynamics}), almost surely over the random $X$ as $n,d \goto \infty$,
\begin{align*}
    \corr{\frac{\dif^2}{\dif t^2} (\bar{W}\tran \bar{W}), \frac{\dif^2}{\dif t^2} (\bar{W}\tran K \bar{W})} \rightarrow 1~. 
\end{align*}
\end{theorem}
The proof follows from the mean term of $\Kn$ giving the leading order terms in the \CNFC{} derivatives calculation here, and $\one\one\tran$ is a rank-$1$ projector. The proof is deferred to Appendix~\ref{app: ommitted proofs}.

We clarify that although we take infinite width, we are not necessarily in the NTK regime, as we allow for arbitrary scaling of the weights, including the $\mu$P parametrization \citep{MuP}. Hence, our setting allows for feature learning.

In the next section we will construct
a dataset which interpolates between adversarial
and aligned eigenstructure to demonstrate the range of possible values the derivative correlations can take, and theoretically predict their values.

% \section{Theoretically predicting the centered NFC}
% \label{sec: theoretical predictions}

% In this section, we show how our decomposition for the time derivatives 
% of the centered NFC in early times can be used to predict its value theoretically. 

\subsection{Early time \CNFC{} dynamics in co-scaling regime}
Another interesting limit is the linear co-scaling regime, where $n, d, k\to\infty$
with $n/d\equiv \psi_{1}$ and $k/d\equiv \psi_{2}$. Here, we show it is possible to theoretically predict the correlation of the derivatives of the centered NFM and AGOP.

In this regime, we expect the
following two properties \citep{AdlamLinearize2}, which we state as assumptions. First, our key quantities can be written as traces of products of large random matrices; as the dimensions of all such matrices increase, we expect these traces to converge to their average values. This property is known as \textit{self-averaging}.
\begin{assumption}[Self-averaging] We assume that the expected traces appearing in the NFC across initializations are equal to the traces themselves.
\end{assumption}

The elements of weight matrices $W$ are 
drawn from independent Gaussians (i.i.d. within each matrix).
If the data $X$ were also standard Gaussian, we would be able to apply free probability --- the noncommutative analog of classical independence ---
to compute traces of matrix products involving analytic functions of $W$ and $X$ in the limit of large
dimensions \citep{FreeProbabilityBook}. To apply free probability for more general $X$, we require the following assumption: 
\begin{assumption}[Asymptotic freedom of initial parameters, and input-label pairs] 
We assume $W$ and $X$ are asymptotically freely independent. Further, we assume the labels $Y(X)$ are asymptotically
free of $W$ (but not $X$).
\end{assumption}
For example, the labels in the student teacher setups of \citet{AdlamLinearize1, AdlamLinearize2} satisfy this condition.

From Proposition \ref{prop: Early NFA dynamics}, we know that for MSE loss the correlation of NFM/AGOP time derivatives
under gradient flow at initialization can be written as:
\vspace{-0.1cm}
\begin{equation}
\label{eq: nfa equation}
\begin{split}
    &\corr{\dW\tran \dW, \dW\tran K \dW} =\, \trace{X\tran Y \Kn Y X X\tran Y \Kn^2 Y X} \cdot \trace{(X\tran Y \Kn Y X)^2}^{-1/2}\cdot \trace{(X\tran Y \Kn^2 Y X)^2}^{-1/2}.
\end{split}
\end{equation}
In our high-dimensional limit, we expect that, for example, the first term can be given as
\begin{equation}
     \lim_{n,d,k\to\infty}\trace{X\tran Y \Kn Y X X\tran Y \Kn^2 Y X}= \lim_{n,d,k\to\infty}\Exp{\theta}{\trace{X\tran Y \Kn Y X X\tran Y \Kn^2 Y X}}.
\end{equation}
We can decompose the average as
\begin{equation}
\label{eqn: nfa, mean-variance}
\begin{split}
    &\Exp{\theta}{\trace{X\tran Y \Kn Y X X\tran Y \Kn^2 Y X}} =
    \trace{X\tran Y \Exp{\theta}{\Kn} Y X X\tran Y \Exp{\theta}{\Kn^2} Y X}
    + \trace{\Cov{X\tran Y \Kn Y X, X\tran Y \Kn^2 Y X}}
\end{split}
\end{equation}
with similar decompositions for the denominator term.

Therefore if the statistics of $\Kn$ can be understood as a function of $X$, we can compute the correlation in this linear triple-scaling
limit. We focus for now on a one-hidden layer quadratic network (similar to the previous section) to avoid
the branching of terms that arises in more complicated networks. We provide some additional analysis of the first term
in Equation \ref{eqn: nfa, mean-variance} for more complicated architectures in Appendix~\ref{sec: mean predictions}.

\subsection{Exact predictions with one hidden layer and quadratic activations}

\begin{figure*}[!h]
    \centering
    \includegraphics[scale=0.5]{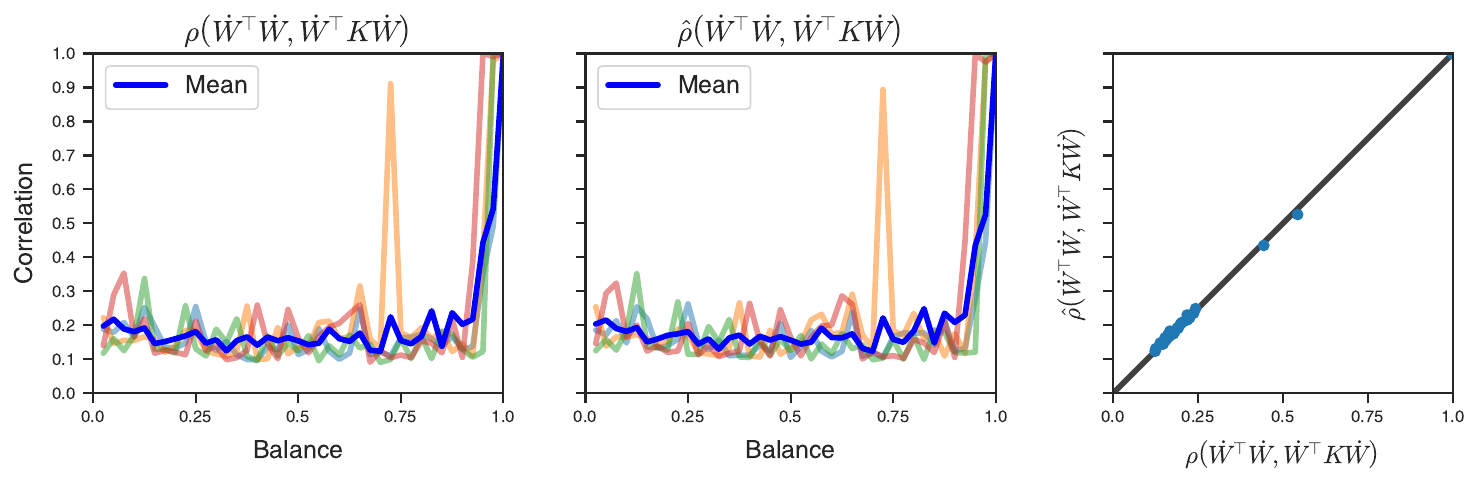}
    \caption{Predicted versus observed correlation of the second derivatives of centered $\NFM$ and $\AGOP$ on the alignment reversing dataset. Different shaded color curves correspond to four different seeds for the dataset. The solid blue curve is the average over all data seeds. The rightmost sub-figure is a scatter plot of the predicted versus observed correlations of these second derivatives, with one point for each balance value. We instantiate the dataset in the proportional regime where  width, input dimension, and dataset size are all equal to $1024$.}
    \label{fig: Predicted NFA calcs}
\end{figure*}

Concretely, we study a neural network $f$ which can be written as $f(x) = a\tran (Wx)^2$, where $a$ is the readout layer, and the square is element-wise. Let $\Amat = (X\tran Y X)^2$, and $\Bmat = X\tran X$, and $\Rmat = W\tran\diag{a^{2}}W$. Then,
the numerator of Equation \ref{eqn: nfa, mean-variance} can be written as:
\begin{equation*}
\trace{X\tran Y \Kn Y X X\tran Y \Kn^2 Y X} = \trace{\Amat \Rmat \Amat \Rmat \Bmat \Rmat}~.
\end{equation*}

We reiterate our assumptions and compute this trace (as well as those for the denominator terms) in Appendix~\ref{app: free probability calculation} using
standard results from random matrix theory. These calculations show us that the correlation is determined by traces of powers of $\Amat$ and $\Bmat$ by the calculations in Section~\ref{app: free probability calculation}, which are properties of the input-label pairs, and $\Rmat$, which is specific to the architecture and initialization statistics.

\paragraph{Manipulating the \CNFC{}} To numerically explore the validity of the random matrix theory calculations, we developed a method to generate datasets with different values of $\corr{\dW\tran \dW, \dW\tran K \dW}$. We construct a random dataset called the \textit{alignment reversing} dataset, parameterized by a \textit{balance} parameter $\gamma \in (0,1]$ to adversarially 
disrupt the NFA near initialization in the regime that width $k$, input dimension $d$, and dataset size $n$ are all equal ($n=k=d=1024$). By Proposition~\ref{prop: expected nfm for quadratic}, for the aforementioned neural 
architecture, the expected second derivative of the centered NFM satisfies, $\Exp{}{\dW\tran \dW} = X\tran Y \Exp{}{\Kn} Y X = (X\tran Y X)^2$, while the expected second derivative of the 
centered AGOP, $\Exp{}{\dW\tran K \dW} = X\tran Y \Exp{}{\Kn^2} Y X$, has an additional component $X\tran Y X \cdot X\tran X \cdot X\tran Y X$. Our construction exploits this difference in that 
$X\tran X$ becomes adversarially unaligned to $X\tran YX$ as the balance parameter decreases. 

The construction exploits that we can manipulate $X\tran Y X \cdot X\tran X \cdot X\tran Y X$ freely of the NFM using a certain choice of $Y$. We design the dataset such that this AGOP-unique term is close to identity, while the NFM second derivative has many large off-diagonal entries, leading to low correlation between the second derivatives of the NFM and AGOP.

In our experiment, we sample multiple random datasets with this construction and compute the predicted and observed correlation of the second derivatives of the centered NFC at initialization. For specific details on the construction, see Appendix~\ref{app: balance dataset}.

We observe in Figure~\ref{fig: Predicted NFA calcs} that the centered NFC predicted with random matrix theory closely matches the observed values, across individual four random seeds and for the average of the correlation across them. Crucially, a single neural network is used across the datasets, confirming the validity of the self-averaging assumption. The variation in the plot across seeds come from randomness in the sample of the data, which cause deviations from the adversarial construction.

\section{Increasing the centered contribution to the NFC strengthens feature learning}
\label{sec: learning speeds}

Our theoretical and experimental work has established that gradient based training leads
to alignment of the weight matrices to the PTK feature covariance.
This process is driven by the \CNFC{}. Therefore, one path towards improving the neural feature correlations is to increase the contribution of the \CNFC{} to the
dynamics - as measured, for example, by the centered-to-uncentered, or, C/UC, ratio
$\trace{\Wb\tran \Wb \Wb\tran K \Wb} \trace{W\tran W W\tran K W}^{-1}$. 
When this ratio is large, the \CNFC{} contributes significantly in magnitude to the \UCNFC{}, indicating successful feature learning as measured by the uncentered NFC.
We will discuss how small initialization promotes feature learning, and design an optimization rule, \OptName{}, that increases the C/UC ratio and drives the value of the \UCNFC{} to $1$.

\subsection{Feature learning and initialization}

One factor that modulates the level of feature learning is the scale of the initialization in each layer. In our experiments training networks with unmodified gradient descent, we observe that the centered NFC will increasingly dominate the uncentered NFC with decreasing initialization (third panel, first row, Figure~\ref{fig: NFA measurements, speed fixing}). Further, as the centered NFC increases in contribution to the uncentered quantity, the strength of the UC-NFC, and to a lesser extent,
the C-NFC increases (Figure~\ref{fig: NFA measurements, speed fixing}). The decreases in correspondence between these quantities is also associated with a decrease in the feature
quality of the NFM (Appendix~\ref{app: initialization and feature learning}), for the chain monomial task.

For fully-connected networks with homogeneous activation functions, such as ReLU, and no normalization layers,
decreasing initialization scale is equivalent to decreasing the scale of the outputs, since we can write $f(Wx) = a^{-p} f(a W x)$ for any scalar $a$ for any homogenous activation
$f$. This in turn is equivalent to increasing the scale of the labels. Therefore, decreasing initialization forces the weights to change more in order to fit the labels, leading to more change in $F$ from its initialization. Conceptually, this aligns with the substantial line of empirical and theoretical evidence that increasing initialization scale or output scaling transitions training between the lazy and feature learning regimes \citep{ChizatLazyTraining, InitLazyTraining, TempCheck, LargeInitGrokking}.

This relationship suggests that small initialization can be broadly applied to increase the change in \NFM{}, and the value of the \UCNFC{}. However, this may not
be ideal; for example, if the activation function is differentiable at $0$, small initialization
leads to a network which is approximately linear.
This may lead to low expressivity unless the learning dynamics can increase the weight magnitude.

\subsection{\OptName{}}

\begin{figure*}[t]
    \centering
    \includegraphics[scale=0.55]{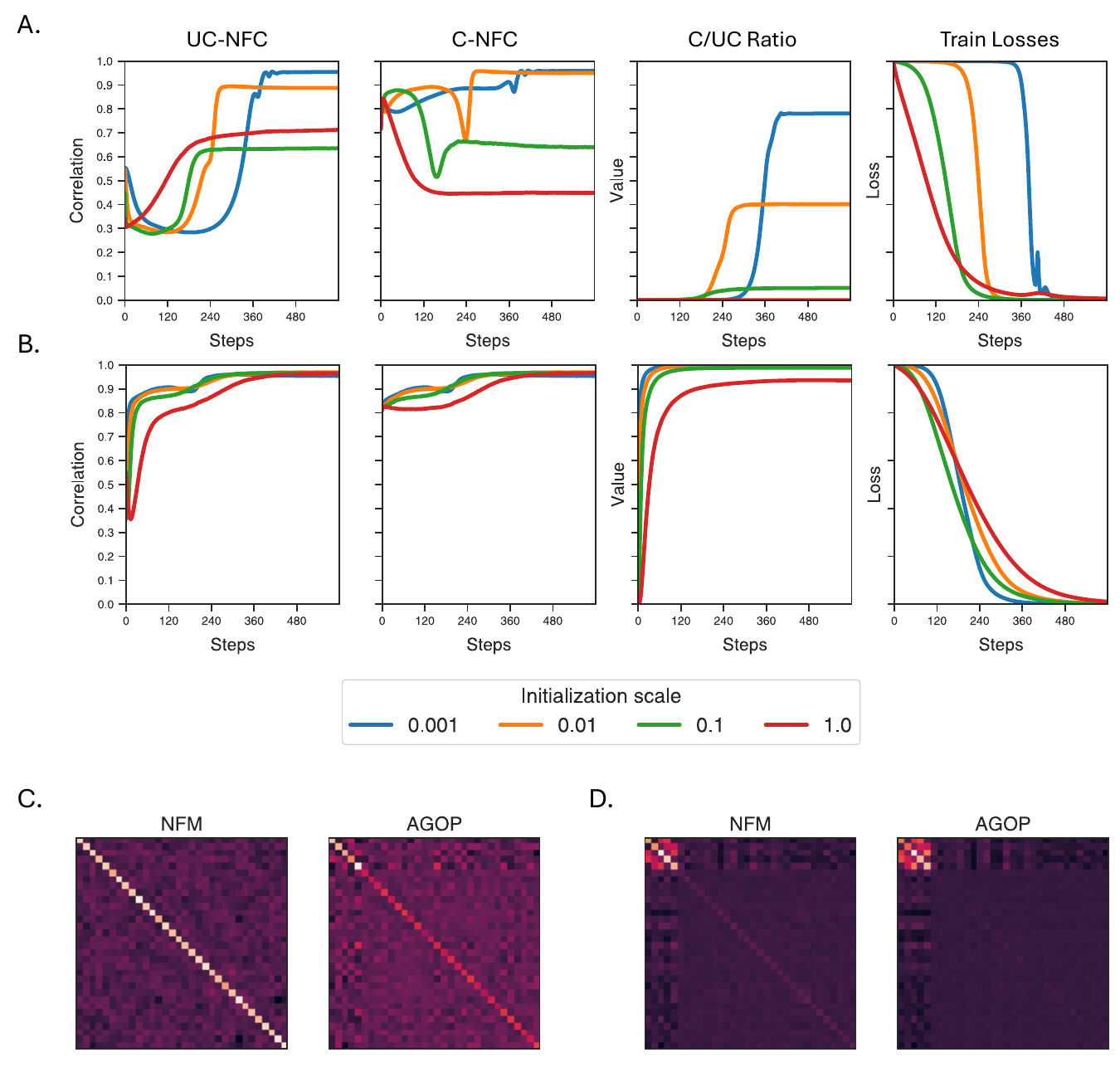}
    \caption{The effect of \OptNameAcronym{} on C/UC neural feature correlations and feature learning on the chain monomial task. In the first two rows, we plot the uncentered and centered NFA for the first layer weight matrix as a function of initialization, (A) with standard training and (B) with \OptNameAcronym. We consider a two hidden layer network with ReLU activations, where we set $C_0 = 500$, and $C_1 = C_2 = 0.002$. The third column shows the ratio of the unnormalized \CNFC{} to the \UCNFC{}: ${\trace{\bar{W}\tran \bar{W} \bar{W}\tran K \bar{W}} \cdot  {\trace{W\tran W W\tran K W}}^{-1}}$. The fourth column shows the training loss. In the third row, we plot the NFM and AGOP from a trained network with (C) standard training and (D) with \OptName{} with fixed initialization scale of $1.0$.}%
    \label{fig: NFA measurements, speed fixing}
\end{figure*}

We can instead design an intervention which can increase feature learning without the need to decrease the initialization scale. We do so by fixing the learning \textit{speed} layerwise to constant values, which causes the \CNFC{} to dominate the \UCNFC{} dynamics. For weights at layer $\ell$ and learning rate $\eta > 0$, we introduce \textit{\OptName{} (SLO)}, which is characterized by the following update rule,
\begin{align*}
    W^{(\ell)}_{t+1} \leftarrow W^{(\ell)}_{t} - \eta \cdot C_\ell \cdot \frac{\grad_{W^{(\ell)}_{t}} \Lo}{\|\grad_{W^{(\ell)}_{t}} \Lo\|}
\end{align*}
where the hyperparameter $C_\ell \geq 0$ controls
the amount of learning in layer $\ell$.
We expect this rule to increase the strength of the \UCNFC{} in layers where $C_\ell$ is large relative to $C_m$ for $m \neq \ell$, as $W^{(\ell)}$ will be forced to change significantly from initialization. By forcing the weights in a particular layer $\ell$ to have fixed learning speeds, these weights will have fixed updates sizes at every epoch, regardless of the loss. We downscale the speeds in other layers $\ell \neq m$ to prevent training instability. As a result, $\norm{W_0^{(\ell)}}\norm{W_t^{(\ell)}}^{-1} \goto 0$ for large $t$, causing the centered and uncentered NFC to coincide for this layer. 

We demonstrate the effects of SLO on the chain-monomial task (Equation~\eqref{eq:chain monomial}). We found that fixing the learning speed to be high in the first layer and low in the remaining layers causes the ratio of the unnormalized \CNFC{} to the \UCNFC{} to become close to $1$ across initialization scales (Figure~\ref{fig: NFA measurements, speed fixing}). The same result holds for the SVHN dataset (see Appendix~\ref{app: real datasets}). We note that this intervention can be applied to target underperforming layers and improve generalization in deeper networks (see Appendix~\ref{app: more SLO experiments} for details).

We observe that both the \CNFC{} and the
\UCNFC{} become close to $1$ after training with SLO, independent
of initialization scale
(Appendix~\ref{app: initialization and feature learning}). Further, the quality of the features learned, measured by the similarity
of $\NFM$ and $\AGOP$ to the true 
EGOP, significantly improve and are more 
similar to each other with SLO, even with large initialization. In contrast,
in standard training the 
\UCNFA{} fails to develop with large initialization, as $\NFM$ 
resembles identity (no feature learning), while $\AGOP$ only slightly captures the relevant features.

\OptName{} is a step toward the design of optimizers that improve generalization and training times through maximizing feature learning.

\section{Discussion}
\paragraph{Analyzing more general settings} In principle our analysis can
be extended to a larger sample size and activation functions with uncentered derivatives. Both of these settings require analyzing more terms in the Taylor expansion of the PTK matrix.
This is more complicated than studying the loss, as our most complex calculations require understanding degree $8$ polynomials of the inputs
even in the simplest case of a one-hidden layer network (as opposed to degree $4$ to understand the loss).
We may also want to understand the case that $n = d^\ell$ for integer $\ell$ (i.e. without the logarithmic factor we consider). The appropriate Taylor expansions in these cases are discussed in \citet{theo2022}, and will likely require additional structure on the coefficients of the target function $f^*$ in the basis of spherical harmonics.

\paragraph{Centered NFC through training}

Our theoretical analysis in this work shows that the PTK feature covariance at initialization has a relatively simple structure in terms of the weights and the data. To predict the NFC later in training, we will likely need to account for the change in this matrix. One should be able to predict this development 
at short times by taking advantage of the fact that
eigenvectors of the Hessian change slowly during
training \citep{HessianInertia}. An alternative approach would be to use a quadratic model for neural network dynamics \citep{AtishQuadratic, LibinQuadratic} which can capture dynamics of the empirical NTK (and therefore the PTK). One promising avenue to study the time dynamics of the PTK could be to adapt the results of \citet{wang2024nonlinear}, which study how the effects of feature learning can propagate to the conjugate kernel for neural networks. It may also be useful to unify our notion of alignment with the observations of \citet{alignment2023}, where they find that the weights align with the averaged outer products of the loss gradients.

\paragraph{\OptName{}} There are prior works that implement differential learning rates across layers \citep{AdaptiveLRfinetuning, LayerSpecificLRs}, though these differ from differential learning \textit{speeds} as in \OptNameAcronym{}, in which the L2-norm of the gradients are fixed, and not the scaling of the gradients. Our intervention is more similar to the LARS \citep{LARS} and LAMB \citep{LAMB} optimizers that fix learning speeds to decrease training time in ResNet and BERT architectures. However, these optimizers fix learning speeds to the norm of the weight matrices, while our intervention sets learning speeds to free hyperparameters in order to increase feature learning, as measured by the NFC. We also point out that as an additional explanation for the success of \OptNameAcronym{} in our experiments, the PTK feature covariance may change slowly (or just in a single direction) with this optimizer, because the learning rate is small for later layers, allowing the first layer weights to align with the PTK.

Other works consider interventions that vary the strength of feature learning across layers. The authors of \citet{MuP} down-scale the final layer weights so that the weights in earlier layers must move significantly in $\ell_\infty$ norm to fit the data. The SLO intervention is more extreme than this approach in that SLO overrides the natural GD dynamics to exactly set the weight movement rates. In \citet{chizat2024featurespeedformulaflexible} the authors consider adjusting learning rates in each layer so as to force the activation vectors to substantially change across all layers. Instead, SLO sets learning speeds of the weights in each layer so that some weights move much faster than others, depending on the layer where we want to maximize feature learning.

\paragraph{Adaptive and stochastic optimizers}

We demonstrate in this work that the dataset (by our construction in Section~\ref{sec: cnfc theory}) and optimizer (with \OptNameAcronym{} in Section~\ref{sec: learning speeds}) play a significant role in the strength of feature learning. Important future work would be to understand other settings where significant empirical differences exist between optimization choices. In particular, analyzing the NFC may clarify the role of gradient batch size and adaptive gradient methods in generalization \citep{CatapultsFeatureLearning}. 

\section{Acknowledgements}

We thank Lechao Xiao for detailed feedback on the manuscript. We also thank Jeffrey
Pennington for helpful discussions. This work used the programs (1) XSEDE (Extreme science and engineering discovery environment) which is supported by NSF grant numbers ACI-1548562, and (2) ACCESS (Advanced cyberinfrastructure coordination ecosystem: services \& support) which is supported by NSF grants numbers \#2138259, \#2138286, \#2138307, \#2137603, and \#2138296. Specifically, we used the resources from SDSC Expanse GPU compute nodes, and NCSA Delta system, via allocations TG-CIS220009.

\bibliographystyle{unsrtnat}
\bibliography{aux/references, aux/related-work}

\newpage
\appendix
\onecolumn

\clearpage
\section{Glossary}
\label{app: glossary}
\begin{enumerate}[leftmargin=*]
    \item \textbf{Expected gradient outer product (EGOP).} Defined with respect to a target function $f^*$ and an input data distribution $\mu$ over data in $d$ dimensions. $$\EGOP(f^*, \mu) \equiv \Exp{x\sim\mu}{\grad_x f^*(x) \grad_x f^*(x)\tran} \in \Real^{d \times d}.$$
    \item \textbf{Average gradient outer product (AGOP).} Defined with respect to a predictor $f$ and a set of inputs $\{x^{(1)}, \ldots, x^{(n)}\}$. $$\mathrm{AGOP}(f, \{x^{(\alpha)}\}_\alpha) \equiv \frac{1}{n} \sum_{\alpha=1}^n \dfdx{f(x^{(\alpha)})}{x} \dfdx{f(x^{(\alpha)})}{x}\tran \in \Real^{d \times d}.$$
    In the context of a deep neural network, we write the AGOP $\AGOP_\ell$ with respect to the inputs of a given layer $\ell$ as $$\AGOP_\ell \equiv \frac{1}{n} \sum_{\alpha=1}^n \dfdx{f(x^{(\alpha)})}{h_\ell} \dfdx{f(x^{(\alpha)})}{h_\ell}\tran.$$
    \item \textbf{Neural feature matrix (NFM).} Given a neural network $f$ with a weight matrix $W^{(\ell)}$ at layer $\ell$, the NFM $\NFM_\ell$ is defined as, $$\NFM_\ell \equiv (W^{(\ell)})\tran W^{(\ell)}.$$
    \item \textbf{Neural feature correlation (NFC).} Defined with respect to a layer of a neural network, this is the correlation $\corr{\NFM_\ell, \AGOP_\ell}.$
    \item \textbf{Neural feature ansatz (NFA).} This refers to the statement that the NFC will have correlation approximately equal to $1$.
    \item \textbf{Pre-activation tangent kernel (PTK).} This is the kernel function corresponding, defined with respect to a neural network and a layer $\ell$, that evaluates two inputs $x, z$ to $\dfdx{f(x)}{h_\ell} \cdot \dfdx{f(z)}{h_\ell} \in \Real$.
    \item \textbf{PTK feature covariance.} For a given layer of a neural network and a dataset $X$, the PTK feature covariance $K^{(\ell)}$ is defined as $K^{(\ell)} \equiv \dfdx{f(X)}{h_\ell}\tran \dfdx{f(X)}{h_\ell} \in \Real^{k_\ell \times k_\ell}$.
    \item \textbf{Centered NFC (C-NFC).} Consider a layer $W$ of a neural network with PTK feature covariance $K$ with respect to that layer. Let $\Wb \equiv W - W_0$, where $W_0$ are weights at initialization. Then, the C-NFC is equivalent to the correlation $\corr{\Wb\tran \Wb, \Wb\tran K \Wb}$.
    \item \textbf{Uncentered NFC (UC-NFC).} This quantity is identical to the NFC defined above.
\end{enumerate}

\clearpage
\section{Connection of the PTK to the empirical NTK}
\label{app: ptk and entk}
We note the following connection between the PTK entries and the empirical NTK (ENTK). In particular, the PTK is a significant component of the ENTK.
\begin{proposition}[Pre-activation to neural tangent identity]
\label{prop: PTK->NTK}
Consider a depth-$L$ neural network $f(x)$ with inputs $X_\ell \in \Real^{n \times k_{\ell}}$ to  weight matrices $W^{(\ell)} \in \Real^{k_{\ell} \times k_{\ell}}$ for $\ell \in [L]$. Consider the empirical NTK where gradients are taken only with respect to $W^{(\ell)}$ evaluated between two points $x$ and $z$ denoted by $\ENTK_\ell(x_0, z_0) = \inner{\dfdx{f(x_0)}{W^{(\ell)}}, \dfdx{f(z_0)}{W^{(\ell)}}}$. Then for all $\ell$, $\ENTK_\ell$ satisfies,
\begin{align*}
    \ENTK_\ell(x_0, z_0) = \Kn^{(\ell)}(x_0, z_0) \cdot x_\ell\tran z_\ell~.
\end{align*}
\end{proposition}
\begin{proof}[Proof of Proposition~\ref{prop: PTK->NTK}]
Note that $\dfdx{f(x_0)}{W^{(\ell)}} = \dfdx{f(x_0)}{h_\ell} x_0\tran \in \Real^{k_\ell \times k_{\ell}}$. Then,
\begin{align*}
    \ENTK_\ell(x_0,z_0) &= \trace{ \dfdx{f(x_0)}{h_\ell} x_\ell\tran z_\ell  \dfdx{f(z_0)}{h_\ell}\tran}\\
    &= \dfdx{f(x_0)}{h_\ell}\tran  \dfdx{f(z_0)}{h_\ell} \cdot x_\ell\tran z_\ell\\
    &= \Kn^{(\ell)}(x_0,z_0) \cdot x_\ell\tran z_\ell~.
\end{align*}
\end{proof}

\section{Additional proofs and statements}
\label{app: ommitted proofs}

\begin{theorem*}[Maximum \CNFC{}]
Suppose the data $X$ are sampled uniformly at random from $\S^{d-1}$ and the labels are generated by $f^*$ satisfying Assumptions~\ref{assump: bounded labels} and \ref{assump: linear component}. Suppose we train $f$ with MSE loss and initialize the network so that the initial outputs are $0$ by either Method~\ref{method 1} or Method~\ref{method 2} above. Assume the activation function $\phi$ satisfies Assumption~\ref{assump: diff activation}. We consider the regime that $\omega_d\round{d\log^{2\ell+1}d} \leq n \leq o_d\round{d^{2-\delta}}$ for some $\delta > 0$, and width $k \rightarrow \infty$. Then, at $t=0$ in training (in the setting of Proposition~\ref{prop: Early NFA dynamics}), almost surely over the random $X$ as $n,d \goto \infty$,
\begin{align*}
    \corr{\frac{\dif^2}{\dif t^2} (\bar{W}\tran \bar{W}), \frac{\dif^2}{\dif t^2} (\bar{W}\tran K \bar{W})} \rightarrow 1~. 
\end{align*}
\end{theorem*}
\begin{proof}[Proof of Theorem~\ref{thm: high c-nfc}]
Applying Proposition~\ref{prop: Early NFA dynamics} and using that we train with MSE and zero initial outputs, we have
\begin{align*}
    \frac{1}{2} \frac{\dif^2}{\dif t^2} (\bar{W}\tran \bar{W}) = X\tran Y \Kn Y X,\quad
    \frac{1}{2} \frac{\dif^2}{\dif t^2} (\bar{W}\tran \bar{W}) = X\tran Y \Kn^2 Y X~.
\end{align*}
Given that $k \rightarrow \infty$, $\Kn$ is equal to a deterministic kernel matrix conditioned on the inputs $X$ (i.e. does not depend on the sampled initial weights). We then use that we are in the proportional limit to approximate the kernel matrix $\Kn$ by its Taylor expansion \citep{ElKarouiLinearize, LuLinearize, ba2019generalization, theo2022}. Namely, we have that with probability $1 - o_d(1)$,
\begin{align}
\label{eq: linearization}
    \Kn = \gamma \one\one\tran + \frac{\alpha}{d} XX\tran + \mu I + o_d(1) \cdot \Delta~,
\end{align}
where $\gamma = h_d(0) + \frac{1}{d} h''_d(0)$, $\alpha = h'(0)$, $\mu = h(1) - h(0) - h'(0)$, and $\Delta$ with $\|\Delta\|_2 = 1$ is a matrix. In general, we will write standalone asymptotic variables to indicate matrices $\Delta$ of that order spectral norm.

By the linearization in equation~\eqref{eq: linearization} and the structure of $XX\tran$, we will show,
\begin{align*}
    \Kn = \Theta(1) \cdot \one\one\tran  + \wt{O}\round{1} \cdot \Delta_1,\quad \Kn^2 = \Theta(d) \cdot \one\one\tran  + \wt{O}\round{d} \cdot \Delta_2~,
\end{align*}
where $\Delta_1, \Delta_2$ have spectral norm $1$. To prove this, first note that $\frac{1}{n} X\tran X = I + o(1)$ as uniform data on the sphere are $O(1)$-subgaussian and have identity covariance. Therefore, as $n = \wt{\Theta}(d)$, and the spectra of $X\tran X$ and $XX\tran$ are identical, we have that $\|\frac{1}{d} XX\tran\| = \|\frac{1}{d} X\tran X\| = \wt{O}(1)$.

We now analyze the numerator in the correlation (using the normalized trace). To simplify notation we write $\inner{A,B} \equiv \trace{A\tran B}$ for matrices $A,B$. The numerator is then:
\begin{align*}
    \inner{X\tran Y \Kn Y X, X\tran Y \Kn^2 Y X} &=  \wt{\Theta}(d) \cdot  \trace{(X\tran yy\tran X)^2}\\
    &+ \wt{O}(d) \cdot \trace{(X\tran Y \Delta_1 Y X) (X\tran yy\tran X)}\\
    &+ \wt{O}(d) \cdot \trace{(X\tran yy\tran X) (X\tran Y \Delta_2 Y X)}\\
    &+ \wt{O}(d) \cdot \trace{(X\tran Y \Delta_1 Y X) (X\tran Y \Delta_2 Y X)}\\
\end{align*}
As the labels are bounded, $\|Y \Delta_1 Y\|, \|Y \Delta_2 Y\| = \wt{O}(1)$. Therefore, by the sub-multiplicative property of the spectral norm and that $\|X\tran\|\|X\| = \|X\tran X\|$, we have $\|X\tran Y \Delta_1 Y X\| \leq \|Y \Delta_1 Y\| \|X\tran X\| = \wt{O}(1) \cdot \|X\tran X\|$ and $\|X\tran Y \Delta_2 Y X\| = \wt{O}(1) \cdot \|X\tran X\|$. Finally, again applying $\|X\tran X\| =\wt{O}(d)$, 
\begin{align*}
    \trace{(X\tran Y \Delta_1 Y X) (X\tran yy\tran X)} \leq \wt{O}(1) \cdot \|X\tran X\| \cdot \|X\tran y\|^2 = \wt{O}(d) \cdot \|X\tran y\|^2,
\end{align*}
and,
\begin{align*}
    \trace{(X\tran yy\tran X) (X\tran Y \Delta_2 Y X)} \leq \wt{O}(d) \cdot \|X\tran y\|^2.
\end{align*}
The fourth trace then satisfies,
\begin{align*}
    \trace{(X\tran Y \Delta_1 Y X) (X\tran Y \Delta_2 Y X)} \leq \wt{O}(1) \cdot \|X\tran X\|^2 = \wt{O}(d^2)
\end{align*}
Combining the four traces and using that $\trace{(X\tran yy\tran X)^2} = \trace{(y\tran XX\tran y)^2} = \|X\tran y\|^4$,
\begin{align*}
    \inner{X\tran Y \Kn Y X, X\tran Y \Kn^2 Y X} &= \wt{\Omega}(d) \cdot \|X\tran y\|^4 + \wt{O}(d^2) \cdot \|X\tran y\|^2 + \wt{O}(d^3)~.
\end{align*}
Recall by Assumption~\ref{assump: linear component}, we have $\|\Exp{x}{x f^*(x)}\| = \Omega(1)$. We apply Lemma~\ref{lemm: sub gaussian vector concentration} for our choice of $n$ so that $\|\frac{1}{n}X\tran y - \Exp{}{xf^*(x)}\| = o(1)$ w.p. $1-o(1)$. As a consequence, we have $\|X\tran y\| = \wt{\Omega}(d)$ w.h.p. Therefore, the leading term in the numerator of the correlation is that of $\trace{(X\tran yy\tran X)^2}$ - the term due to the interaction of the mean terms of $\Kn$ and $\Kn^2$.

Meanwhile the first denominator term decomposes as follows:
\begin{align*}
    \inner{X\tran Y \Kn Y X, X\tran Y \Kn Y X} &=  \wt{\Theta}(1) \cdot  \trace{(X\tran yy\tran X)^2}\\
    &+ \wt{O}(1) \cdot \trace{(X\tran Y \Delta_1 Y X) (X\tran yy\tran X)}\\
    &+ \wt{O}(1) \cdot \trace{(X\tran Y \Delta_1 Y X)^2}\\
\end{align*}
By a similar argument to the numerator term, we observe that the leading sub-term for the first denominator term is $\wt{\Theta}(1) \cdot \trace{(X\tran yy\tran X)^2}$.

Further, the second denominator terms decomposes as follows:
\begin{align*}
    \inner{X\tran Y \Kn^2 Y X, X\tran Y \Kn^2 Y X} &=  \wt{\Theta}(d^2) \cdot  \trace{(X\tran yy\tran X)^2}\\
    &+ \wt{O}(d^2) \cdot \trace{(X\tran Y \Delta_2 Y X) (X\tran yy\tran X)}\\
    &+ \wt{O}(d^2) \cdot \trace{(X\tran Y \Delta_2 Y X)^2}\\
\end{align*}
By a similar argument, the second denominator term has leading sub-term $\Theta(d^2) \cdot \trace{(X\tran yy\tran X)^2}$.

Putting the numerator and denominator terms together, we have the correlation has the following asymptotics:
\begin{align*}
    \corr{\frac{\dif^2}{\dif t^2} (\bar{W}\tran \bar{W}), \frac{\dif^2}{\dif t^2} (\bar{W}\tran K \bar{W})} &= \frac{ \wt{\Theta}(d)\cdot \trace{(X\tran yy\tran X)^2} + \wt{O}(d^3)}{\trace{\wt{\Theta}(1) \cdot (X\tran yy\tran X)^2 + \wt{O}(d^2)}^{1/2} \cdot \trace{\wt{\Theta}(d^2) \cdot (X\tran yy\tran X)^2 + \wt{O}(d^4)}^{1/2} } \\
    &\rightarrow 1~, 
\end{align*}
completing the proof.
\end{proof}

\begin{lemma}
\label{lemm: sub gaussian vector concentration}
Suppose $|f^*(x)| \leq c \cdot \log^\ell(d)$ with probability $1-o(1)$ for constant $c > 0$, and the data are sub-Gaussian with constant parameter $K$ and some $\ell > 0$. Then, $\|\frac{1}{n}X\tran y - \Exp{}{xf^*(x)}\| = o(1)$ w.p. $1-o(1)$, provided $n = \omega(d \log^{2\ell+1}(d))$. 
\end{lemma}
\begin{proof}
As the data $x$ are $K$-sub-Gaussian for a universal constant $K$, and the labels are bounded by $c\log^\ell(d)$, we have that $xy$ is sub-Gaussian with parameter $M=Kc\log^\ell(d)$. Therefore, the empirical expectation of $xy$, $\frac{1}{n} \sum_{i=1}^n x_i y_i$ is sub-Gaussian with parameter $M/\sqrt{n}$. 
Directly applying Lemma~1 in \citet{subgaussian_note}, we see that,
\begin{align*}
    \mathrm{Pr}\round{\norm{\frac{1}{n} \sum_{i=1}^n x_i y_i - \Exp{}{xy}} > t} < 2\exp\round{-\frac{t^2 n}{2cM^2 d}}~,
\end{align*}
for a universal constant $c>0$. Therefore, provided $n = \omega(d \log^{2\ell+1}(d))$, we have $\norm{\frac{1}{n} \sum_{i=1}^n x_i y_i - \Exp{}{xy}} = o(1)$ with probability $1-o(1)$, completing the proof.
\end{proof}

\section{Additional centerings of the NFC}
\label{app: other centerings}

\paragraph{Double-centered NFC} One may additionally center the PTK feature map to understand the co-evolution of the PTK feature covariance and the weight matrices. We consider such a centering that we refer to as the \textit{double-centered} NFC, where we measure $\corr{(\Wb^{(\ell)})\tran \Wb^{(\ell)}, (\Wb^{(\ell)})\tran \bar{K}^{(\ell)} \Wb^{(\ell)}}$, where $\bar{K}^{(\ell)} = \round{\dfdx{f(X)}{h_\ell} - \dfdx{f_0(X)}{h_\ell}}\tran \round{\dfdx{f(X)}{h_\ell} - \dfdx{f_0(X)}{h_\ell}}$, and $f_0$ is the neural network at initialization.

However, the double-centered NFC term corresponds to higher-order dynamics that do not significantly contribute the centered and uncentered NFC (Figure~\ref{fig: Unnormalized NFA}) when initialization is large or for early periods of training. Note however this term becomes relevant over longer periods of training.

\begin{figure}[h]
    \centering
    \includegraphics[scale=0.4]{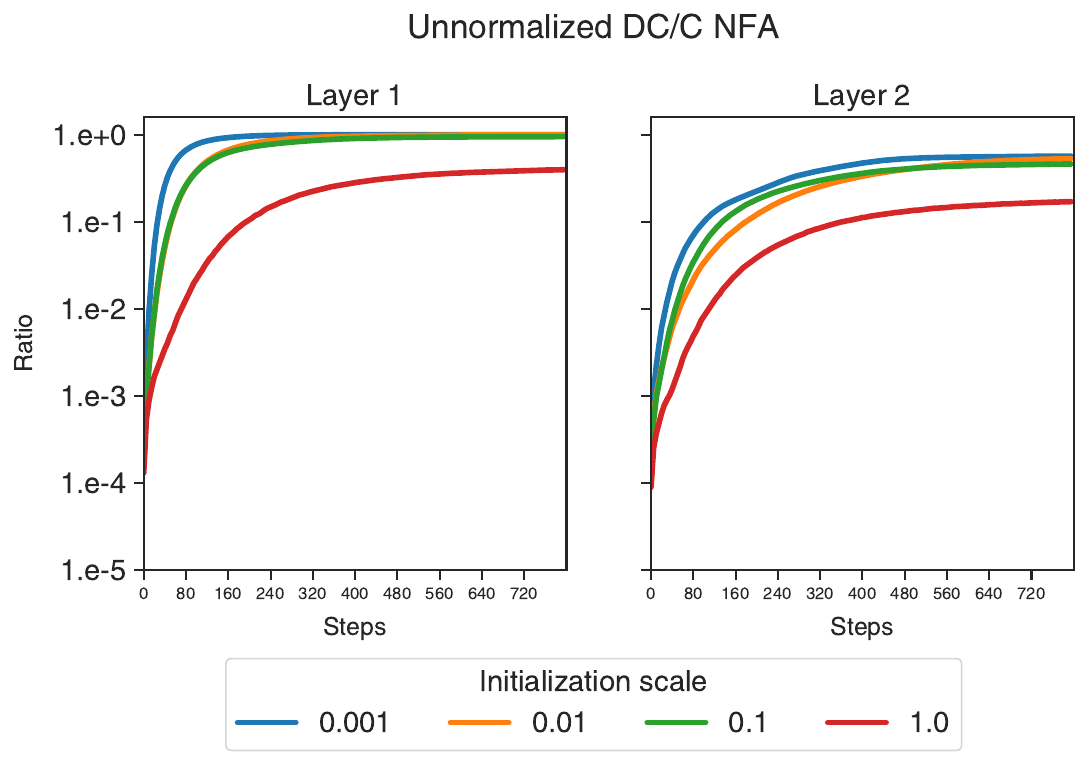}
    \caption{Ratio of the unnormalized double-centered NFC to the centered NFC throughout neural network training. In particular, we plot $\trace{\bar{W}\tran \bar{W} \bar{W}\tran \bar{K} \bar{W}} \cdot {\trace{\bar{W}\tran \bar{W} \bar{W}\tran K \bar{W}}}^{-1}$ throughout training for both layers of a two-hidden layer MLP with ReLU activations.}
    \label{fig: Unnormalized NFA}
\end{figure}

\paragraph{Isolating alignment of the PTK to the initial weight matrix} 

One may also center just the PTK feature map, while substituting the initial weights for $W$ to isolate how the PTK feature covariance aligns to the weight matrices. To measure this alignment, we consider the \textit{PTK-centered} NFC, which is defined as the correlation $\corr{(\Wi^{(\ell)})\tran \Wi^{(\ell)}, (\Wi^{(\ell)})\tran \bar{K}^{(\ell)} \Wi^{(\ell)}}~$, where $\Wi^{(\ell)}$ is the initial weight matrix at layer $\ell$.

However, this correlation decreases through training, indicating that the correlation of these quantities does not drive alignment between the uncentered NFM and AGOP (Figure~\ref{fig: PTK-centered NFA}).

\begin{figure}[h]
    \centering
    \subfloat[Isotropic data.]{\includegraphics[scale=0.5]{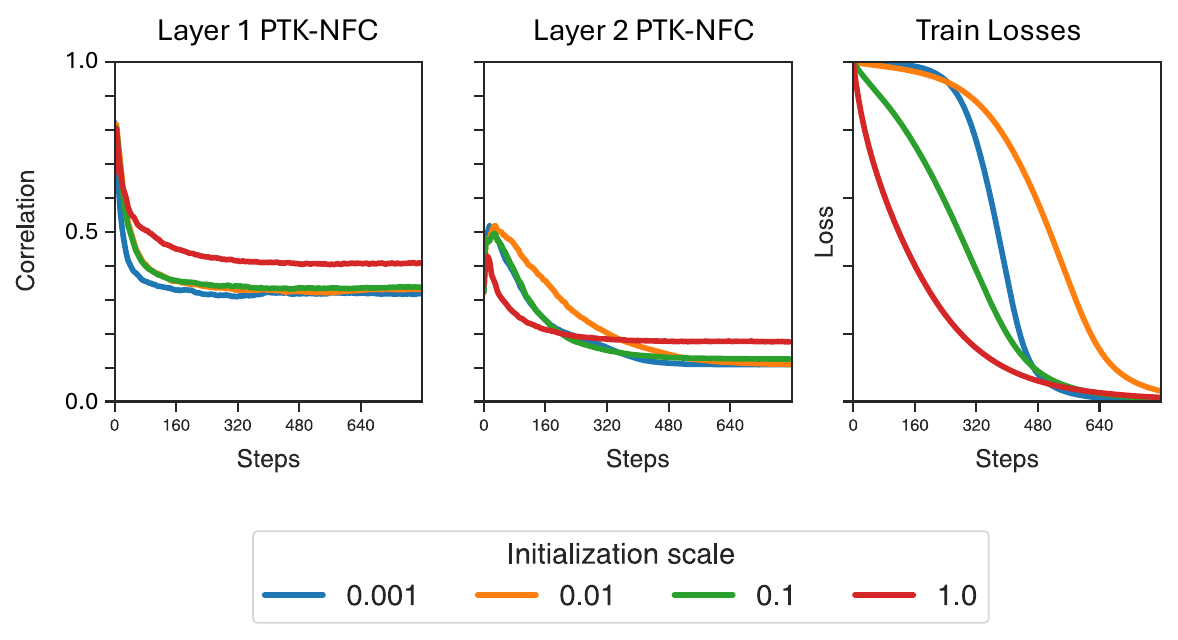}}%
    ~\\
    \subfloat[Data spectrum with decay $\lambda_k \sim \frac{1}{1+k^2}$.]{\includegraphics[scale=0.5]{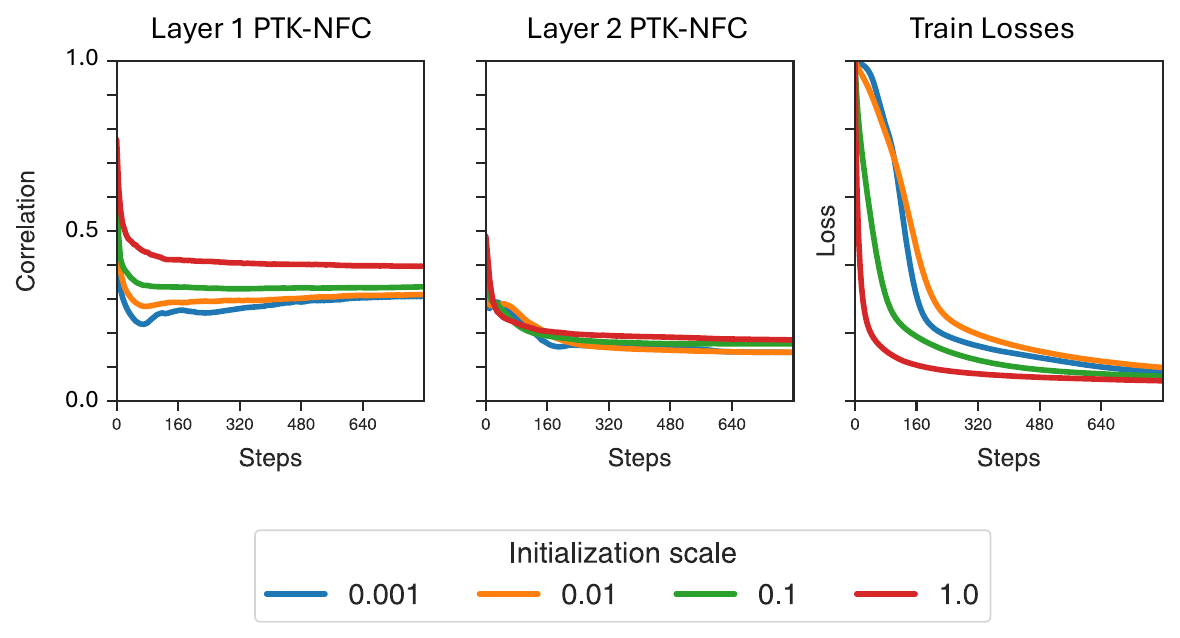}}
    \caption{PTK-centered NFC throughout training for both layers of a two-hidden layer MLP with ReLU activations on Gaussian data with two different spectra.}
    \label{fig: PTK-centered NFA}
\end{figure}

\section{Extending our theoretical predictions to depth and general activations}
\label{sec: mean predictions}

Precise predictions of the C-NFC become more complicated with additional depth and general activation functions. However, we note that the deep C-NFC will remain sensitive to a first-order approximation in which $K$ is replaced by its expectation. We demonstrate that this term qualitatively captures the behavior of the C-NFC for 2 hidden layer architectures with quadratic and, to a lesser extent, ReLU activation functions in Figure~\ref{fig: Depth calcs}. In this experiment, we sample Gaussian data with mean $0$ and covariance with a random eigenbasis. We parameterize the eigenvalue decay of the covariance matrix by a parameter $\alpha$, called the data decay rate, so that the eigenvalues have values $\lambda_k = \frac{1}{1+k^\alpha}$. As $\alpha$ approaches $0$ or $\infty$ the data covariance approaches a projector matrix. 

In this experiment, we see that the data covariance spectrum will also parameterize the eigenvalue decay of $\Exp{}{K}$, allowing us to vary how close the expected PTK matrix (and its dual, the PTK feature covariance) is from a projector, where the NFA holds exactly. We see that for intermediate values of $\alpha$, both the observed and the predicted derivatives of the \CNFC{} decreases in value. 

We plot the observed values in two settings corresponding to different asymptotic regimes. One setting is the proportional regime, where $n=k=d=128$. The other is the NTK regime where $n=d=128$ and $k=1024$. For the quadratic case, as the network approaches infinite width, the prediction more closely matches the observed values. Additional terms corresponding to the nonlinear part of $\phi'$ in ReLU networks, the derivative of the activation function, are required to capture the correlation more accurately in this case.

\begin{figure*}[h]
    \centering
    \subfloat[Quadratic.]{\includegraphics[scale=0.525]{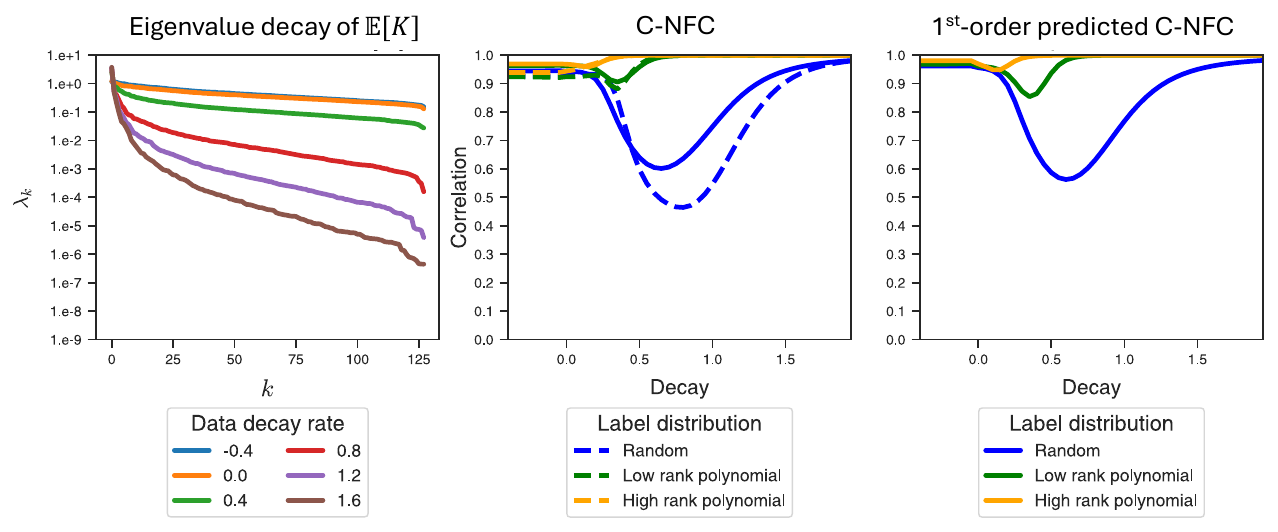}}%
    \\
    \subfloat[ReLU.]{\includegraphics[scale=0.525]{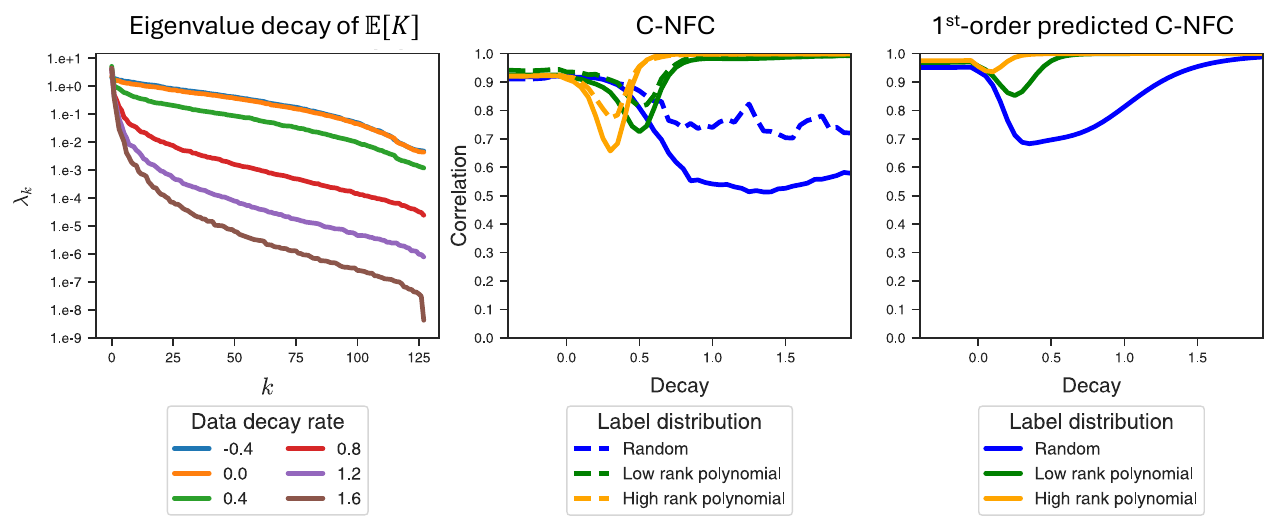}}
    \caption{Observed versus the first-order predicted C-NFC for the input to the first layer of a two hidden layer MLP. The dashed line is neural network width $k=n=d=128$, where $n$ and $d$ are the number of data point and data dimension, respectively, while the solid line uses $n=d=128$ and $k=1024$.}
    \label{fig: Depth calcs}
\end{figure*}

\section{Free probability calculations of C-NFC}
\label{app: free probability calculation}

In order to understand the development of the NFC, we analyze the centered NFC in the limit that learning rate is much smaller than the initialization for a one hidden layer MLP with quadratic activations. We write this particular network as,
\begin{align*}
    f(x) = a\tran (Wx)^2~,
\end{align*}
where $a \in \Real^{1 \times k}$ and $W \in \Real^{k \times d}$, where $d$ is the input dimension and $k$ is the width. In this case, the NFC has the following form,
\begin{align}
\label{eq: quadratic NFA}
    \corr{\NFM, \AGOP} = \frac{\trace{X\tran Y \Kn Y X X\tran Y \Kn^2 Y X}}{\trace{(X\tran Y \Kn Y X)^2}^{-1/2} \trace{(X\tran Y \Kn^2 Y X)^2}^{-1/2}}~,
\end{align}
where $\Kn = XW\tran \diag{a}^2 WX\tran$. 

We assume two properties hold in the finite dimensional case we consider, that will hold asymptotically in the infinite dimensional limit.
\begin{assumption}[Self-averaging] 
We assume that computing the average of the NFC quantities
    across initializations is equal to the quantities themselves in the high-dimensional limit.
\end{assumption}
\begin{assumption}[Asymptotic freeness] 
We assume that the collections $\{X, Y\}$ and $\{W, a\}$
    are asymptotically free with respect to the operator $\Exp{}{\tr(\cdot)}$, where
    $\tr[M] = \frac{1}{n}\sum_{i=1}^{n}M_{ii}$.
\end{assumption}
% \begin{assumption}[Commutativity of expectation]
% We will also make the approximation that the expectation commutes with ratio and square root.
% \begin{equation}
% \Exp{}{\rho(A, B)} = \Exp{}{\frac{\tr(A\tran B)}{\sqrt{\tr(A\tran A)\tr(B\tran B)}}} \approx \frac{\Exp{}{\tr(A\tran B)}}{\sqrt{\Exp{}{\tr(A\tran A)]\expect[\tr(B\tran B)}}}
% \end{equation}
% \end{assumption}

We will compute the expected values of the centered NFC under these assumptions. In the remainder of
the section we will drop the $\Exp{}{\cdot}$ in the trace for ease of notation.

\subsection{Free probability identities}

The following lemmas will be useful: let $\{\bar{C}_{i}\}$ and $\{R_{i}\}$ be freely
independent of each other with respect to $\tr$, with $\tr[\bar{C}_{i}] = 0$. Alternating words
have the following products:
\begin{equation}
\tr[\bar{C}_{1}R_{1}] = 0
\end{equation}
\begin{equation}
\tr[\bar{C}_{1}R_{1}\bar{C}_{2}R_{2}] = \tr[R_{1}]\tr[R_{2}]\tr[\bar{C}_{1}\bar{C}_{2}]
\end{equation}
\begin{equation}
\tr[\bar{C}_{1}R_{1}\bar{C}_{2}R_{2}\bar{C}_{3}R_{3}] = \tr[R_{1}]\tr[R_{2}]\tr[R_{3}]\tr[\bar{C}_{1}\bar{C}_{2}\bar{C}_{3}]
\end{equation}
\begin{equation}
\begin{split}
& \tr[\bar{C}_{1}R_{1}\bar{C}_{2}R_{2}\bar{C}_{3}R_{3}\bar{C}_{4}R_{4}]  = 
\tr[R_{1}]\tr[R_{2}]\tr[R_{3}]\tr[R_{4}]\tr[\bar{C}_{1}\bar{C}_{2}\bar{C}_{3}\bar{C}_{4}] +\\
& \tr[R_{1}]\tr[R_{3}]\tr[\bar{R}_{2}\bar{R}_{4}]\tr[\bar{C}_{1}\bar{C}_{2}]\tr[\bar{C}_{3}\bar{C}_{4}]+\tr[R_{2}]\tr[R_{4}]\tr[\bar{R}_{1}\bar{R}_{3}]\tr[\bar{C}_{2}\bar{C}_{3}]\tr[\bar{C}_{1}\bar{C}_{4}]
\end{split}
\end{equation}
where $\bar{R}_{i} \equiv R_{i}-\tr[R_{i}]$.

Applying these identities to the one hidden layer quadratic case, we use the following definitions:
\begin{equation}
R = W\tran\diag{a^{2}}W,~A = (X\tran Y X)^{2},~B = X\tran X
\end{equation}
Crucially, $R$ is freely independent of the set
$\{A, B\}$. We will also use the notation $\bar{M}$ to indicated the centered version of $M$, $\bar{M} = M-\tr[M]$.

\subsection{Numerator term of NFC}

The numerator in Equation~\eqref{eq: quadratic NFA} is
\begin{equation}
\trace{X\tran Y \Kn Y X X\tran Y \Kn^2 Y X} = \trace{A R A R B R}
\end{equation}
Re-writing $A = \bar{A}+\tr[A]$ and $B = \bar{B}+\tr[B]$ we have:
\begin{equation}
\trace{X\tran Y \Kn Y X X\tran Y \Kn^2 Y X} = \trace{(\bar{A}+\tr[A]) R (\bar{A}+\tr[A]) R (\bar{B}+\tr[B]) R}
\end{equation}
This expands to
\begin{equation}
\begin{split}
\trace{X\tran Y \Kn Y X X\tran Y \Kn^2 Y X} & = \trace{\bar{A}R\bar{A}R\bar{B}R}+2\trace{A}\trace{\bar{A}R\bar{B}R^{2}}+\trace{B}\trace{\bar{A}R\bar{A}R^{2}}\\
&\trace{A}^{2}\trace{\bar{B}R^{3}}+2\trace{A}\trace{B}\trace{\bar{A}R^{3}} + \trace{A}^{2}\trace{B}\trace{R}^{3}
\end{split}
\end{equation}
Using the identities we arrive at:
\begin{equation}
\begin{split}
\trace{X\tran Y \Kn Y X X\tran Y \Kn^2 Y X} & = \tr[A]^2\tr[B]\trace{R^3}\\
& +2\tr[A]\tr[R^2]\tr[R]\trace{\bar{A}\bar{B}}\\
& +\tr[B]\tr[R]\tr[R^2]\trace{\bar{A}^2}\\
& +\tr[R]^{3}\trace{\bar{A}^{2}\bar{B}}
\end{split}
\end{equation}

\subsection{First denominator term of NFC}

The first denominator term in Equation~\eqref{eq: quadratic NFA} is
\begin{equation}
\trace{X\tran Y \Kn Y X X\tran Y \Kn Y X} = \trace{A R  A R}
\end{equation}
This is a classic free probability product:
\begin{equation}
\trace{X\tran Y \Kn Y X X\tran Y \Kn Y X} = \tr[A^{2}]\tr[R]^{2}+\tr[A]^{2}\trace{R^2}-\tr[A]^{2}\tr[R]^{2}
\end{equation}
which can be derived from the lemmas.

\subsection{Second denominator term of NFC}

For the second denominator term of Equation~\eqref{eq: quadratic NFA} we have
\begin{equation}
\trace{X\tran Y \Kn^2 Y X X\tran Y \Kn^2 Y X} = \trace{A R B R A R B R}
\end{equation}
Expanding the first $A$ we have
\begin{equation}
\trace{X\tran Y \Kn^2 Y X X\tran Y \Kn^2 Y X} = \trace{\bar{A} R B R A R B R}+\tr[A]\trace{R^{2} B R A R B }
\end{equation}
Next we expand the first $B$:
\begin{equation}
\begin{split}
\trace{X\tran Y \Kn^2 Y X X\tran Y \Kn^2 Y X} & = \trace{\bar{A} R \bar{B} R A R B R}+\tr[B]\trace{\bar{A} R^2 A R B R}\\
& +\tr[A]\tr[B]\trace{R^{3} A R B }+\tr[A]\trace{R^{2} \bar{B} R A R B }
\end{split}
\end{equation}
The next $A$ gives us
\begin{equation}
\begin{split}
\trace{X\tran Y \Kn^2 Y X X\tran Y \Kn^2 Y X} & = \trace{\bar{A} R \bar{B} R \bar{A} R B R}+2\tr[A]\trace{\bar{A} R \bar{B} R^2 B R}+\tr[B]\trace{\bar{A} R^2 \bar{A} R B R}\\
& +2\tr[A]\tr[B]\trace{R^{3} \bar{A} R B }+\tr[A]^{2}\tr[B]\trace{R^{4} B }+\tr[A]^{2}\trace{R^{2} \bar{B} R^{2} B }
\end{split}
\end{equation}
Expanding the final $B$ we have
\begin{equation}
\begin{split}
\trace{X\tran Y \Kn^2 Y X X\tran Y \Kn^2 Y X} & = \trace{\bar{A} R \bar{B} R \bar{A} R \bar{B} R}+2\tr[B]\trace{\bar{A} R \bar{B} R \bar{A} R^2}+2\tr[A]\trace{\bar{A} R \bar{B} R^2 \bar{B} R}\\
& +4\tr[A]\tr[B]\trace{R^{3} \bar{A} R \bar{B} }+2\tr[A]\tr[B]^2\trace{R^{4} \bar{A} }+2\tr[A]^2\tr[B]\trace{R^{4} \bar{B} }\\
& +\tr[A]^{2}\tr[B]^2\tr[R^4]+\tr[A]^{2}\tr[R^{2}\bar{B}R^{2}\bar{B}]+\tr[B]^{2}\tr[R^{2}\bar{A}R^{2}\bar{A}]
\end{split}
\end{equation}
Now all terms are in the form of alternating products from the lemma.
This means we can factor out the non-zero traces of the other terms.
Simplifying we have:
\begin{equation}
\begin{split}
\trace{X\tran Y \Kn^2 Y X X\tran Y \Kn^2 Y X} & = \tr[R]^{4}\trace{(\bar{A}\bar{B})^{2}}+2\tr[R]^{2}(\tr[R^2]-\tr[R]^2)\tr[\bar{A}\bar{B}]^{2}\\
&+2\tr[R]^{2}\tr[R^2]\left(\tr[B]\trace{\bar{A}^{2}\bar{B}}+\tr[A]\trace{\bar{A}\bar{B}^{2}}\right)\\
& +4\tr[A]\tr[B]\tr[R^{3}]\tr[R]\trace{\bar{A}\bar{B}}+\tr[A]^{2}\tr[B]^2\tr[R^4]\\
& +\tr[A]^{2}\tr[\bar{B}^{2}]\tr[R^{2}]^{2}+\tr[\bar{A}]^{2}\tr[B^{2}]\tr[R^{2}]^{2}
\end{split}
\end{equation}

All terms of the NFC are now in terms of traces of the matrices $A$, $B$, and $R$ and functions on each term separately. The matrices $A$ and $B$ are determined by the data, while the moments of the eigenvalues of $R$ are determined by the initialization distribution of the weights in the neural network, and neither training nor the data.

\section{Alignment reversing dataset}
\label{app: balance dataset}

The data consists of a mixture of two distributions from which two subsets of the data $X^{(1)}$ and $X^{(2)}$ are sampled from, and is parametrized by a balance parameter $\gamma \in (0,1]$ and two variance parameters $\epsilon_1, \epsilon_2 > 0$. The subset $X^{(1)}$ which has label $y_1=1$ and constitutes a $\gamma$ fraction of the entire dataset, is sampled from a multivariate Gaussian distribution with mean $0$ and covariance $\Sigma = \one\one\tran + \epsilon_1 \cdot I~.$ Then the second subset, $X^{(2)}$, is constructed such that $(X^{(2)})\tran X^{(2)} \approx ((X^{(1)})\tran X^{(1)})^{-2}$, and has labels $y_2 = 0$. Then, for balance parameter $\gamma$ sufficiently small, the AGOP second derivative approximately satisfies,
\begin{align*}
    \Exp{}{\dW\tran K \dW} &\sim X\tran Y X X\tran X X\tran Y X = (X^{(1)})\tran X^{(1)} X\tran X (X^{(1)})\tran X^{(1)} \approx (X^{(1)})\tran X^{(1)} (X^{(2)})\tran X^{(2)} (X^{(1)})\tran X^{(1)}\\ &\approx I~,
\end{align*}
In contrast, the NFM second derivative, $\Exp{}{\dW\tran \dW} = (X\tran Y X)^2 = ((X^{(1)})\tran X^{(1)})^2 \approx \Sigma^2$, will be significantly far from identity.

Specifically, we construct $X^{(2)}$ by the following procedure:
\begin{enumerate}
    \item Extract singular values $S_1$ and right singular vectors $U_1$ from a singular-value decomposition (SVD) of $(X^{(1)})\tran X^{(1)}$. 
    \item Extract the left singular vectors $V_2$ from a sample $\tilde{X}_2$ that is sampled from the same distribution as $X^{(1)}$.
    \item Construct $X^{(2)} = V_2 S_1^{-1} U_1\tran$.
    \item Where $X = X^{(1)} \oplus X^{(2)}$, Set $X \leftarrow X + \epsilon_2 Z$, where $Z \sim \mathcal{N}(0, I)$.
    \item Set $y \leftarrow y + 10^{-5}\cdot \one$.
\end{enumerate}

Note that $U_1 S_1^{-1} V_2\tran V_2 S_1^{-1} U_1\tran = U_1 S_1^{-2} U_1\tran = ((X^{(1)})\tran X^{(1)})^{-2}$, therefore, we should set $X^{(2)} = V_2 S_1^{-1} U_1\tran$ to get $(X^{(2)})\tran X^{(2)} = ((X^{(1)})\tran X^{(1)})^{-2}$. Regarding the variance parameters, in practice we set $\epsilon_1 = 0.5$ and $\epsilon_2 = 10^{-2}$.

\begin{proposition}[Expected NFM and AGOP]
\label{prop: expected nfm for quadratic}
For a one hidden layer quadratic network, $f(x) = a\tran (Wx)^2$, with $a \sim \mcN(0,I)$ and $W \sim \frac{1}{\sqrt{k}}  \cdot \mcN(0,I)$,
\begin{align*}
    \Exp{a, W}{\dW\tran \dW} = (X\tran Y X)^2~,
\end{align*}
and,
\begin{align*}
\Exp{a, W}{\dW\tran K \dW} &= 3 \cdot \trace{X\tran X} \cdot (X\tran Y X)^2\\ 
&+ 6 X\tran Y X X\tran X X\tran Y X
\end{align*}
\end{proposition}

\begin{proof}[Proof of Proposition~\ref{prop: expected nfm for quadratic}]
\begin{align*}
    \Exp{}{\dW\tran \dW} &= X\tran Y X \Exp{}{W_0\tran \diag{a}^2 W_0} X\tran Y X\\
    &= (X\tran Y X)^2~.
\end{align*}
Further,
\begin{align*}
    \Exp{}{\dW\tran K \dW} &= X\tran Y \Exp{}{K^2} Y X~.
\end{align*}
We note that,
\begin{align}
    K^2 &= W_0\tran \diag{a}^2 W_0 X\tran X W_0\tran \diag{a}^2 W_0 \\
    &= \sum_{s_1, s_2}^k \sum_{\alpha}^n \sum_{p_1, p_2}^d\\
    &a_{s_1}^2 a_{s_2}^2 W_{s_1,p_1} X_{\alpha, p_1} X_{\alpha,p_2} W_{s_2,p_2} X W_{s_1} W_{s_2}\tran X\tran~.
\end{align}
Therefore, applying Wick's theorem, element $i,j$ of $K^2$ satisfies,
\begin{align*}
    \Exp{}{K^2_{ij}} &= \sum_{s}^k \sum_{\alpha}^n \sum_{p_1, p_2}^d \Exp{}{a_{s}^4 W_{s,p_1} X_{\alpha,p_1} X_{\alpha,p_2} W_{s,p_2} X_i\tran W_{s} W_{s}\tran X_j} = \sum_{s}^k \sum_{\alpha}^n \sum_{p_1, p_2, q_1, q_2}^d\\
    &\Exp{}{a_{s}^4 W_{s,p_1} W_{s,p_2} W_{s, q_1} W_{s, q_2} X_{\alpha,p_1} X_{\alpha,p_2} X_{i,q_1} X_{j,q_2}}\\
    &= 3 \sum_{s}^k \sum_{\alpha}^n \sum_{p_1, p_2, q_1, q_2}^d\\
    &\Big(\Exp{}{W_{s,p_1} W_{s,p_2}}\Exp{}{W_{s, q_1} W_{s, q_2}} + \\
    &\Exp{}{W_{s,p_1} W_{s,q_1}}\Exp{}{W_{s, p_2} W_{s, p_2}} +\\
    &\Exp{}{W_{s,p_1} W_{s,q_2}}\Exp{}{W_{s, p_2} W_{s, q_1}} \Big)\\
    &\cdot X_{\alpha,p_1} X_{\alpha,p_2} X_{i,q_1} X_{j,q_2}\\
    &= 3 \sum_{\alpha}^n \sum_{p_1, p_2, q_1, q_2}^d\\
    &\Big(\delta_{p_1 p_2} \delta_{q_1 q_2} + \delta_{p_1 q_1} \delta_{p_2 q_2} + \delta_{p_1 q_2} \delta_{p_2 q_1}\Big)\\ 
    &\cdot X_{\alpha,p_1} X_{\alpha,p_2} X_{i,q_1} X_{j,q_2}\\
    &= 3\sum_{\alpha}^n \Big(\sum_{p_1, q_1}^d X_{\alpha,p_1} X_{\alpha,p_1} X_{i,q_1} X_{j,q_1}\\ 
    &+ \sum_{p_1, p_2}^d X_{\alpha,p_1} X_{\alpha,p_2} X_{i,p_1} X_{j,p_2}\\ 
    &+ \sum_{p_1, p_2}^d X_{\alpha,p_1} X_{\alpha,p_2} X_{i,p_2} X_{j,p_1}\Big)\\
    &= 3 \cdot \trace{X\tran X} \cdot X_i\tran X_j + 3 \sum_{\alpha}^n X_{\alpha}\tran X_i X_{\alpha}\tran X_{j}\\
    &+ 3\sum_{\alpha}^n X_{\alpha}\tran X_j X_{\alpha}\tran X_{i}\\
    &= 3 \cdot \trace{X\tran X} \cdot X_i\tran X_j + 3 X_i X\tran X X_{j} + 3 X_j X\tran X X_{i}~.
\end{align*}
Finally, we conclude,
\begin{align*}
    \Exp{}{K^2} = 3 \round{\trace{X\tran X} X X\tran + 2 XX\tran X X\tran}~,
\end{align*}
giving the second statement of the proposition.
\end{proof}

\section{Varying the data distribution}
\label{app: nfa, vary data spectrum}
We verify that our observations for isotropic Gaussian data hold even when the data covariance has a significant spectral decay. (Figures \ref{fig: NFA measurements decay 1} and \ref{fig: NFA measurements decay 2}). We again consider Gaussian data that is mean $0$ and where the covariance is constructed from a random eigenbasis. In Figure~\ref{fig: NFA measurements decay 1}, we substitute the eignevalue decay as $\lambda_k \sim \frac{1}{1+k}$, while in Figure~\ref{fig: NFA measurements decay 2}, we use $\lambda_k \sim \frac{1}{1+k^2}$. We plot the values of the \UCNFC{}, \CNFC{}, train loss, and test loss throughout training for the first and second layer of a two hidden layer network with ReLU activations, while additionally varying initialization scale. Similar to Figure~\ref{fig: NFA measurements decay 0}, we observe that the \CNFC{} is more robust to the initialization scale than the \UCNFC{}, and \UCNFC{} value become high through training, while being small at initialization. We see that the test loss improves for smaller initializations, where the value of the \CNFC{} and \UCNFC{} are higher.

\begin{figure}[h]
    \centering
    \includegraphics[scale=0.5]{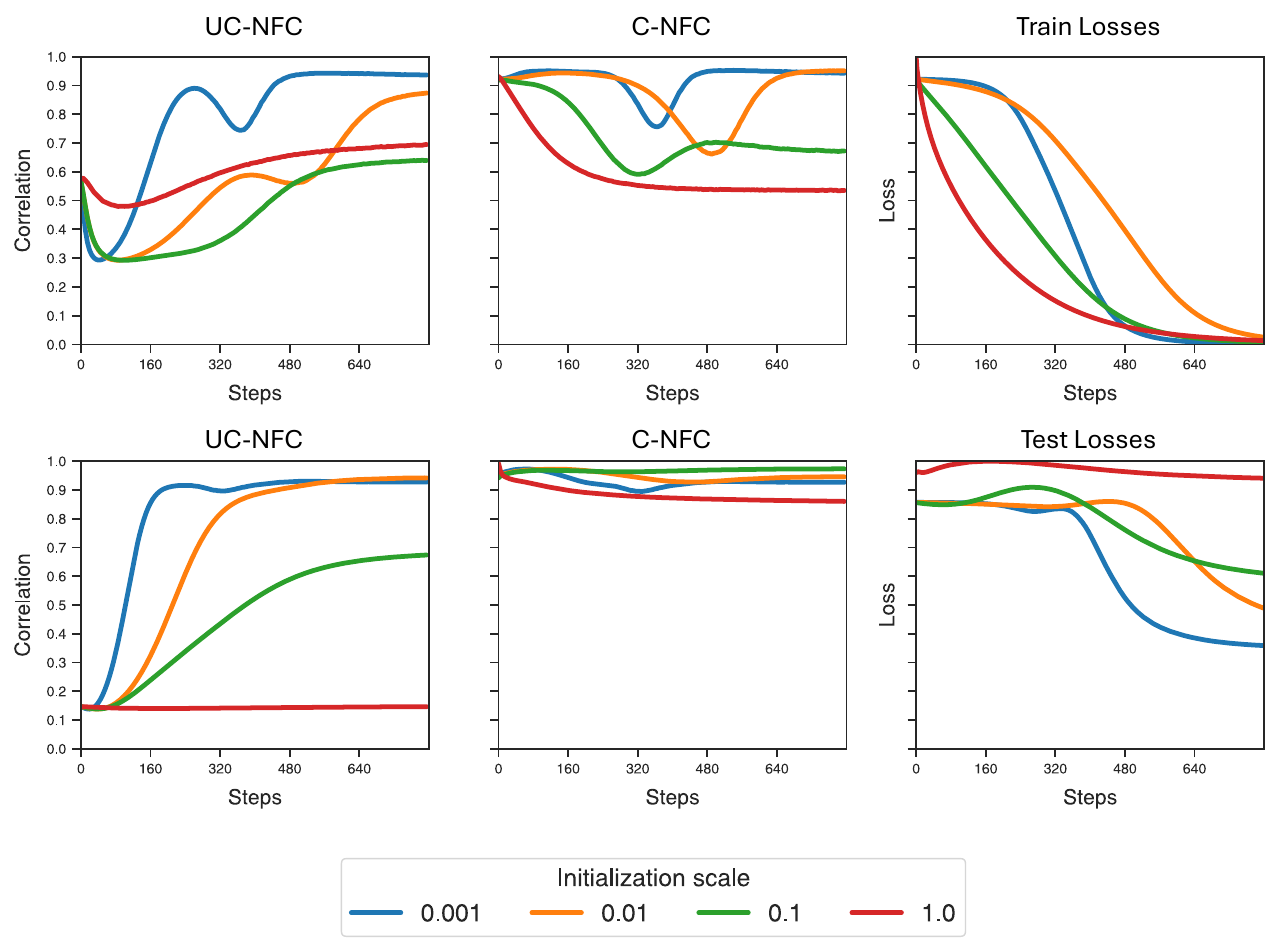}
    \caption{Centered neural feature correlations. Data covariance decay rate $\lambda_k = 1$. Top row is layer 1, bottom row is layer 2. Train (test) losses are scaled by the maximum train (test) loss achieved so that they are between $0$ and $1$.}
    \label{fig: NFA measurements decay 0}
\end{figure}

\begin{figure}[h]
  \centering
    \subfloat{\includegraphics[scale=0.5]{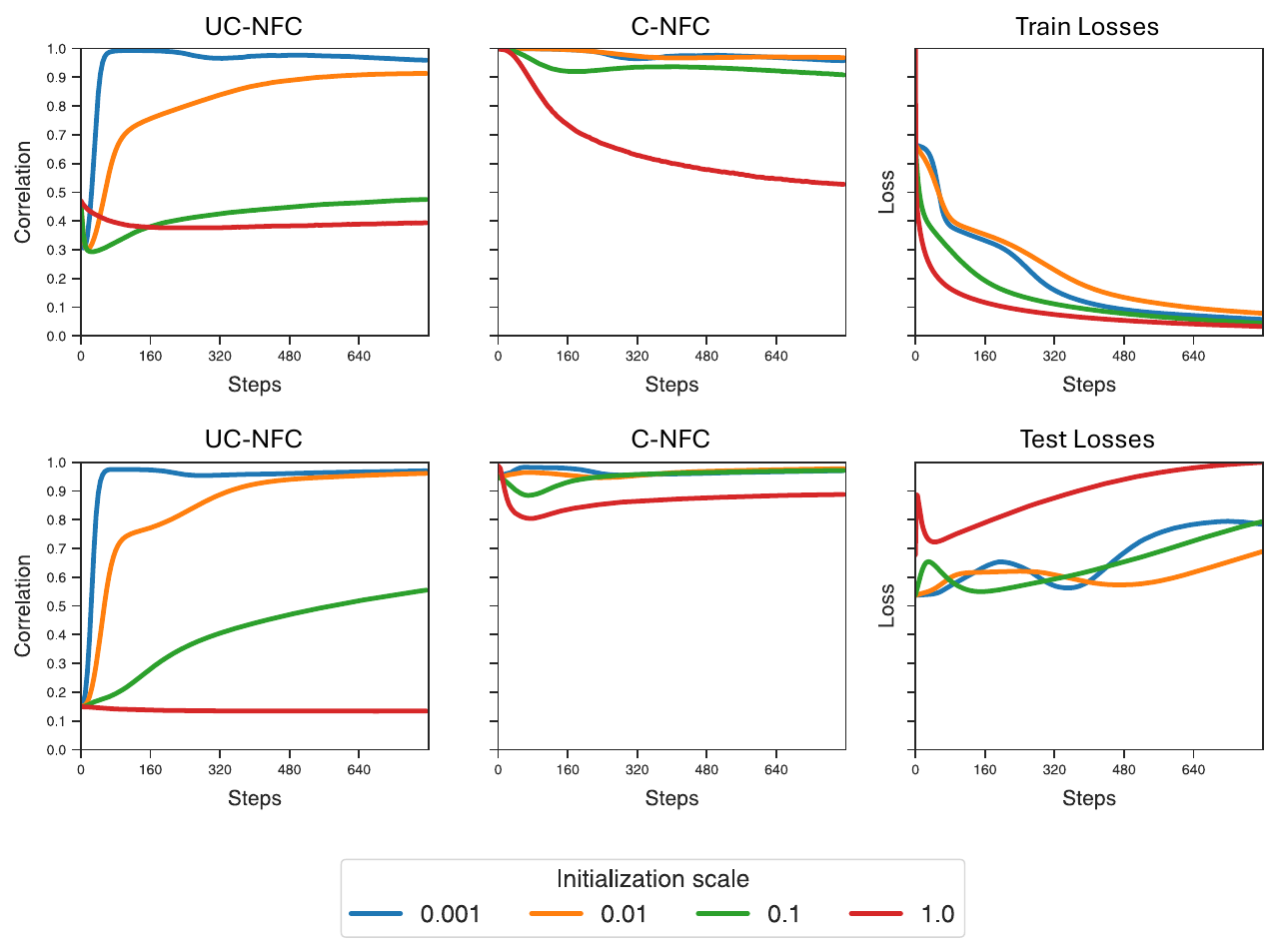}}
    \caption{Centered neural feature correlations. Data covariance decay rate $\lambda_k \sim \frac{1}{1+k}$. Top row is layer 1, bottom row is layer 2. Train (test) losses are scaled by the maximum train (test) loss achieved so that they are between $0$ and $1$.}
    \label{fig: NFA measurements decay 1}
\end{figure}

\begin{figure}[h]
    \centering
    \subfloat{\includegraphics[scale=0.5]{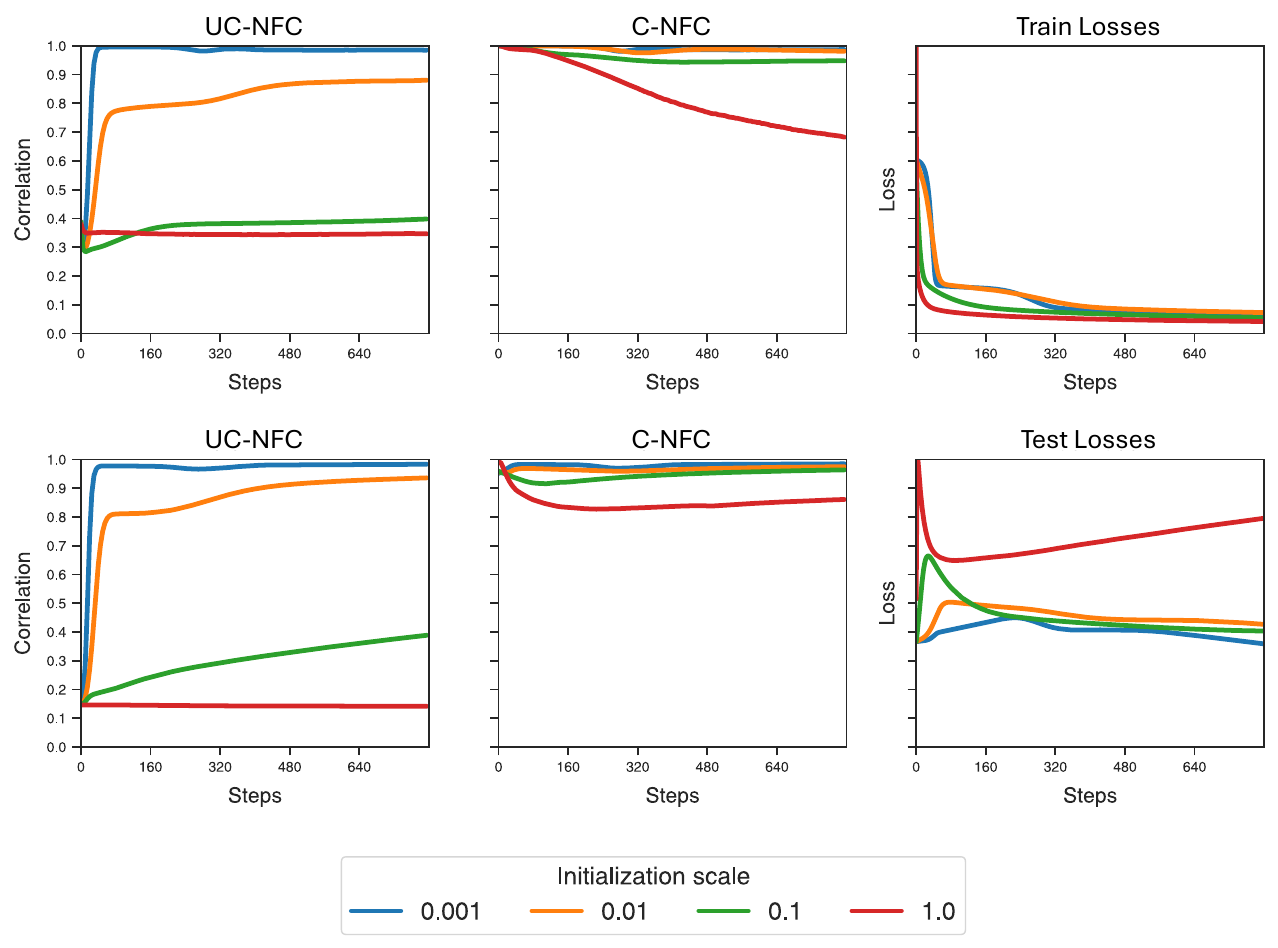}}
    \caption{Centered neural feature correlations. Data covariance decay rate $\lambda_k \sim \frac{1}{1+k^2}$. Top row is layer 1, bottom row is layer 2. Train (test) losses are scaled by the maximum train (test) loss achieved so that they are between $0$ and $1$.}
    \label{fig: NFA measurements decay 2}
\end{figure}

\clearpage

\section{Effect of initialization on feature learning}
\label{app: initialization and feature learning}

We see that when initialization is small, the C-NFC and UC-NFC are high at the end of training with and without fixing the learning speed (Figure~\ref{fig: NFA measurements, speed fixing}). This is reflected by the quality of the features learned by the NFM and the qualitative similarity of the NFM and AGOP at small initialization scale (Figure~\ref{fig: Feature visualization, varying init}). Further, we notice that as we increase initialization, without fixing speeds, the correspondence between the NFM and decreases and the quality of the NFM features decreases (at a faster rate than the AGOP). Strikingly, when learning speeds are fixed, the quality of the features in the AGOP and NFM becomes invariant to the initialization scale. 

\begin{figure*}[h]
    \centering
    \subfloat[NFM, speeds not fixed.]{\includegraphics[scale=0.45]{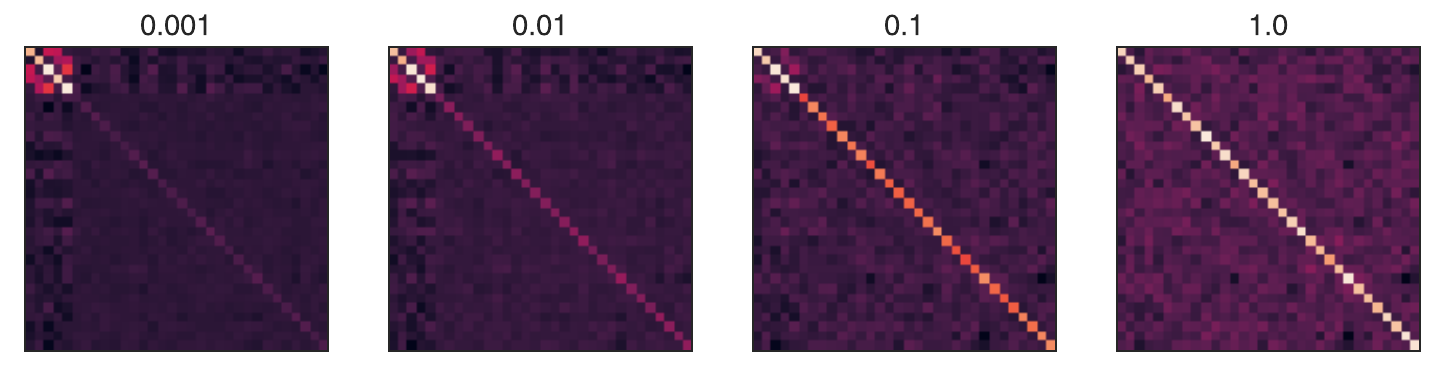}}%
    ~\\
    \subfloat[AGOP, speeds not fixed.]{\includegraphics[scale=0.45]{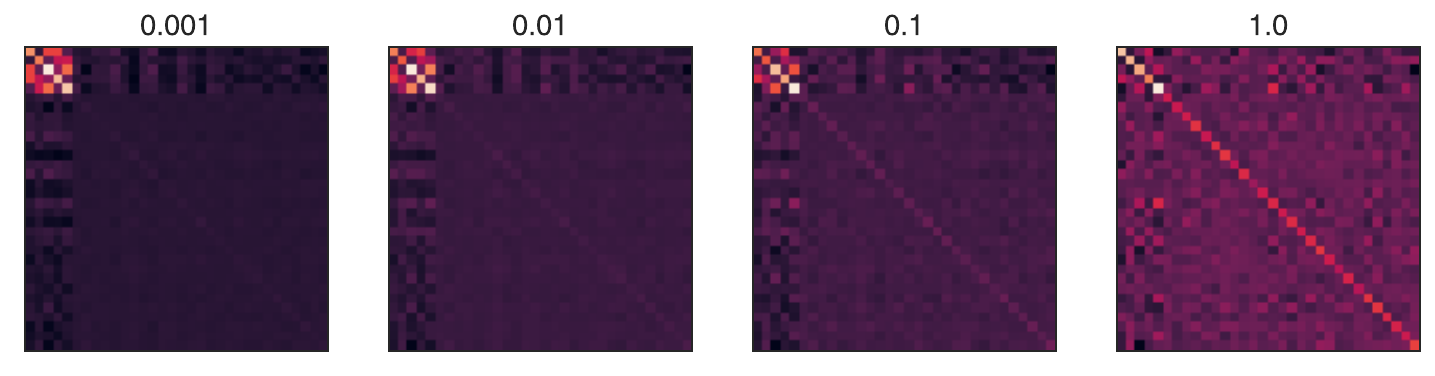}}%
    ~\\
    \subfloat[NFM, speeds fixed.]{\includegraphics[scale=0.45]{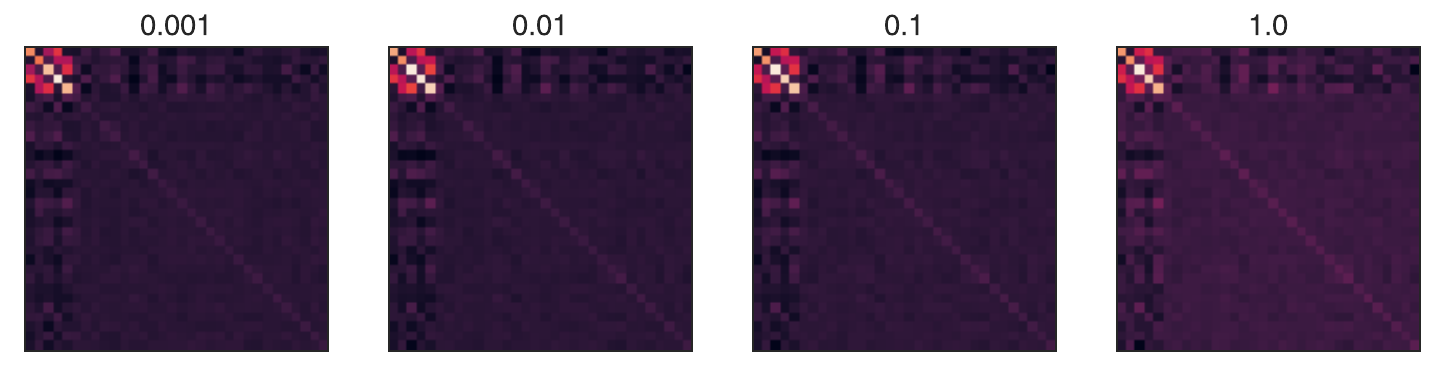}}%
    ~\\
    \subfloat[AGOP, speeds fixed.]{\includegraphics[scale=0.45]{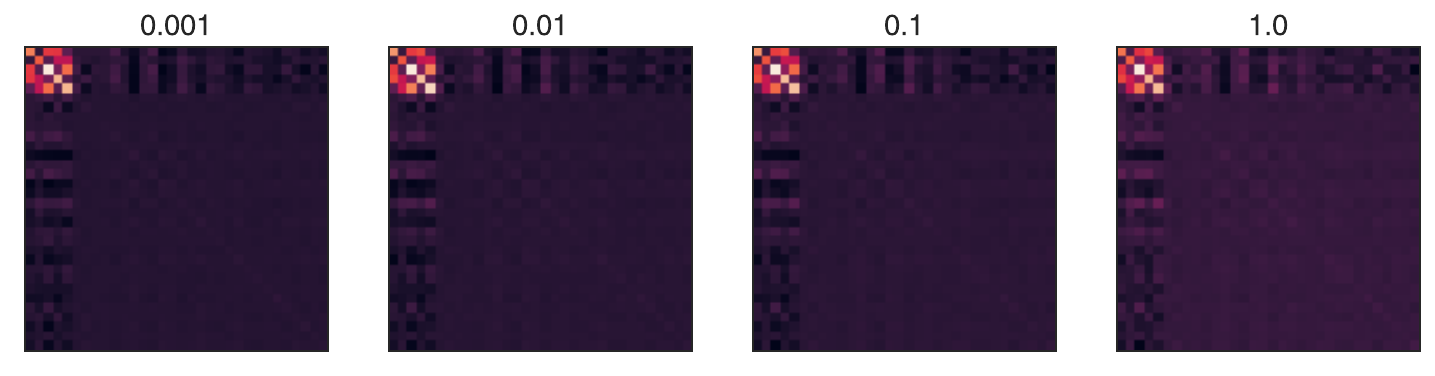}}
    \caption{The NFM at the end of training as a function of initialization scale in the first layer weights, with and without fixing learning speeds. The task again is the chain monomial with rank $r=5$. The title of each plot is the initialization scale of the first layer $s_0$. }
    \label{fig: Feature visualization, varying init}
\end{figure*}

\section{Experimental details}

We describe the neural network training and architectural hyperparameters in the experiments of this paper. Biases were not used for any networks. Further, in all polynomial tasks, we scaled the label vector to have standard deviation $1$.

\paragraph{Corrupted AGOP} For the experiments in Figure~\ref{fig: Corrupted Ansatz}, we used $n=384$ data points, $d=32$, $k=128$ as the width in all layers, isotropic Gaussian data, initialization scale $0.01$ in the first layer and default scale in the second. We used ReLU activations and two hidden layers. For the experiments in Figure~\ref{fig: NFA measurements decay 0},\ref{fig: NFA measurements decay 1},\ref{fig: NFA measurements decay 2},\ref{fig: Unnormalized NFA}, and \ref{fig: PTK-centered NFA}, we used a two hidden layer network with ReLU activations, learning rate $0.05$, 800 steps of gradient decent, and took correlation/covariance measurements every 5 steps. 

\paragraph{C/UC-NFC calculations on real datasets} We describe the experimental details for Figure~\ref{fig: real NFA measurements}. For (A,B) we trained a five layer MLP on the first 50,000 datapoints of \textit{Streetview House Numbers} (SVHN), CIFAR-10, CIFAR-100, STL-10, MNIST, and \textit{German Traffic Sign Recognition Benchmark} (GTSRB) datasets. We used the default PyTorch initialization (scale of $1$) for all layers. We used width 256 in all layers, and trained with SGD with batch size 128. For SVHN and CIFAR, we trained for 150 epochs with learning rate 0.2. For STL-10 and GTSRB we used learning rate 0.1 for 150 epochs. For MNIST we trained for 50 epochs with learning rate 10. For the VGG-11 experiments on CIFAR-10, we used the default architecture from \texttt{torchvision} with batch-norm layers removed. We trained for 50 epochs and learning rate 1. For experiments with a GPT-family model, we adapted the model and dataset from NanoGPT 
(\url{https://github.com/karpathy/nanoGPT}). We used all the default settings and the default Adam optimizer, but used no weight decay, learning rate 5e-3, and removed all dropout layers. We also reduced the number of attention layers to 3 from the original 6. We trained on the Shakespeare characters dataset. 

\paragraph{Alignment reversing dataset} For the experiments in Figure~\ref{fig: Predicted NFA calcs}, we used $k=n=d=1024$ for the width, dataset size, and input dimension, respectively. Further, the traces of powers of $\Rmat$ are averaged over 30 neural net seeds to decouple these calculated values from the individual neural net seeds. The mean value plotted in the first two squares of figure is computed over 10 data seeds. 

\paragraph{\OptNameAcronym{} experiments} For the \OptNameAcronym{} figures (Figures~\ref{fig: Feature visualization, varying init}, \ref{fig: NFA measurements, speed fixing}), we use isotropic Gaussian data, 600 steps of gradient descent. The learning rates are chosen based on initialization scale in the first layer. For initialization scales 1, 0.1, 0.01, and 0.001, we used learning rates $0.03$, $0.1$, $0.2$, $0.4$, respectively. We again used two hidden layers with ReLU activations. We chose $n = 256$, $d=32$, and $k=256$ as the width. We divided the linear readout weights by $0.01$ at initialization to promote feature learning, and
modified \OptNameAcronym{} to scale gradients by $(\epsilon + \|\grad\Lo\|)^{-1}$, rather than just the inverse of the norm of the gradient, for $\epsilon = 0.1$. This technique smooths the training dynamics as the parameters approach a loss minimum, allowing the network to interpolate the labels.

\paragraph{Predictions with depth} For the Deep C-NFC predictions (Figure~\ref{fig: Depth calcs}), we used $n=128$, $d=128$, initialization scale of $1$. The low rank task is just the chain monomial of rank $r=5$. The high rank polynomial task is $y(x) = \sum_{i=1}^d (Qx)_i^2$, where $Q \in \Real^{d \times d}$ is a matrix with standard normal entries.

\paragraph{Figures~\ref{fig: Corrupted Ansatz, CelebA} and \ref{fig: SLO experiments, SVHN}} For the experiments on the SVHN dataset, we train a four hidden layer neural network with ReLU activations, initialization scale $1.0$ in all layers and width $256$. For SVHN, we subset the dataset to 4000 points. We train for 3000 epochs with learning rate 0.2 for standard training, and 0.3 for \OptNameAcronym{}, and take NFC measurements every 50 epochs. For \OptNameAcronym{}, we set $C_0 = 2.5$, $C_1=C_2=0.4$, and relaxation parameter $\epsilon = 0.2$. We pre-process the dataset so that each pixel is mean $0$ and standard deviation $1$. For the experiments on CelebA, we train a two hidden layer network on a balanced subset of 7500 points with Adam with learning rate $0.0001$ and no weight decay. We use initialization scale 0.02 in the first layer, and width 128. We train for 500 epochs. We pre-process the dataset by scaling the pixel values to be between 0 and 1. 

\paragraph{Code availability} We make the code available for the experiments in Figure~\ref{fig: real NFA measurements} available at \url{https://anonymous.4open.science/r/centered_NFA-4795/} (for MLP+VGG).

\newpage
\section{Additional \OptNameAcronym{} experiments}
\label{app: more SLO experiments}

We demonstrate the \OptNameAcronym{} can be applied adaptively to increase the strength of the \UCNFC{} in all layers of a deep network on the chain monomial task of rank $r=3$. We train a three hidden layer MLP with ReLU activations and an initialization scale of $0.1$ by \OptNameAcronym{}, and find that all layers finish at the same high \UCNFC{} (Figure~\ref{fig: Additional SLO Experiments}). Further, this final \UCNFC{} value is higher than the highest \UCNFC{} achieved by any layer with standard training. The generalization loss is also lower with \OptNameAcronym{} on this example, corresponding to better feature learning (through the \UCNFC{}).

\begin{figure*}[h]
  \centering
    \subfloat{\includegraphics[scale=0.625]{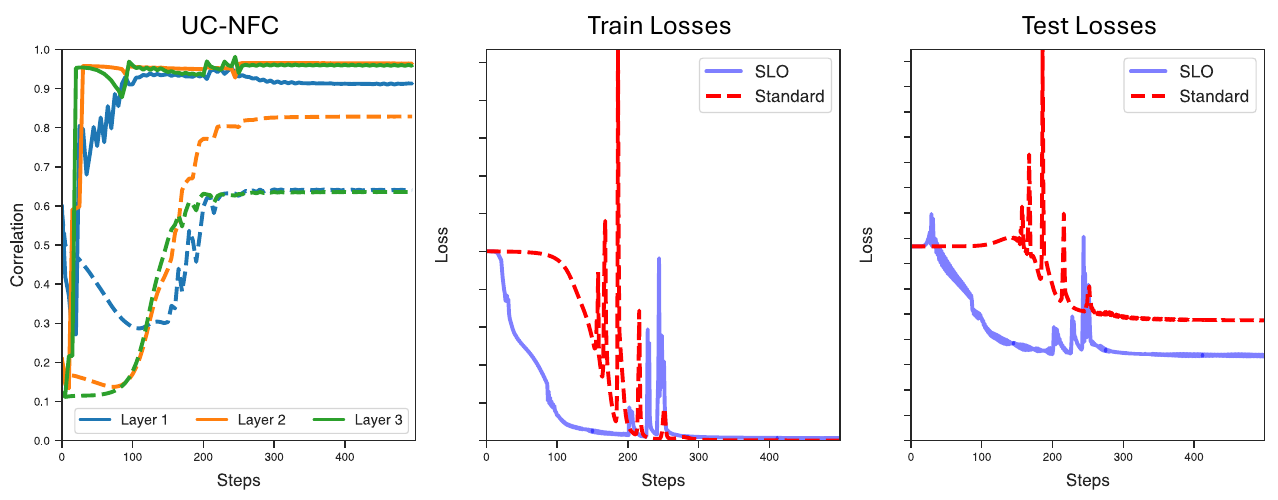}}%
    \caption{Training with \OptNameAcronym{} where the learning speeds are chosen adaptively based on the \UCNFC{} values of all layers. The dashed lines correspond to training with standard GD.}
    \label{fig: Additional SLO Experiments}
\end{figure*}

At every time step we choose $C_i = s$ for the layer $i$ with the smallest \UCNFC{} correlation value, while setting $C_j = s^{-1}$ for all other layers, with $s=20$. We again modify \OptNameAcronym{} by dividing the gradients by $\epsilon \|\grad \Lo\|$ for $\epsilon = 0.01$. The learning rate is set to $0.05$ in \OptNameAcronym{} and $0.25$ for the standard training (gradient descent), and the networks are trained for 500 epochs. We sample $n=256$ points with $d=32$, and use width $k=256$.

\newpage
\section{Additional C-NFC/UC-NFC plots across architectures}
\label{app: C/UC-NFA plots}

\begin{figure*}[h]
    \centering
    \includegraphics[scale=0.6]{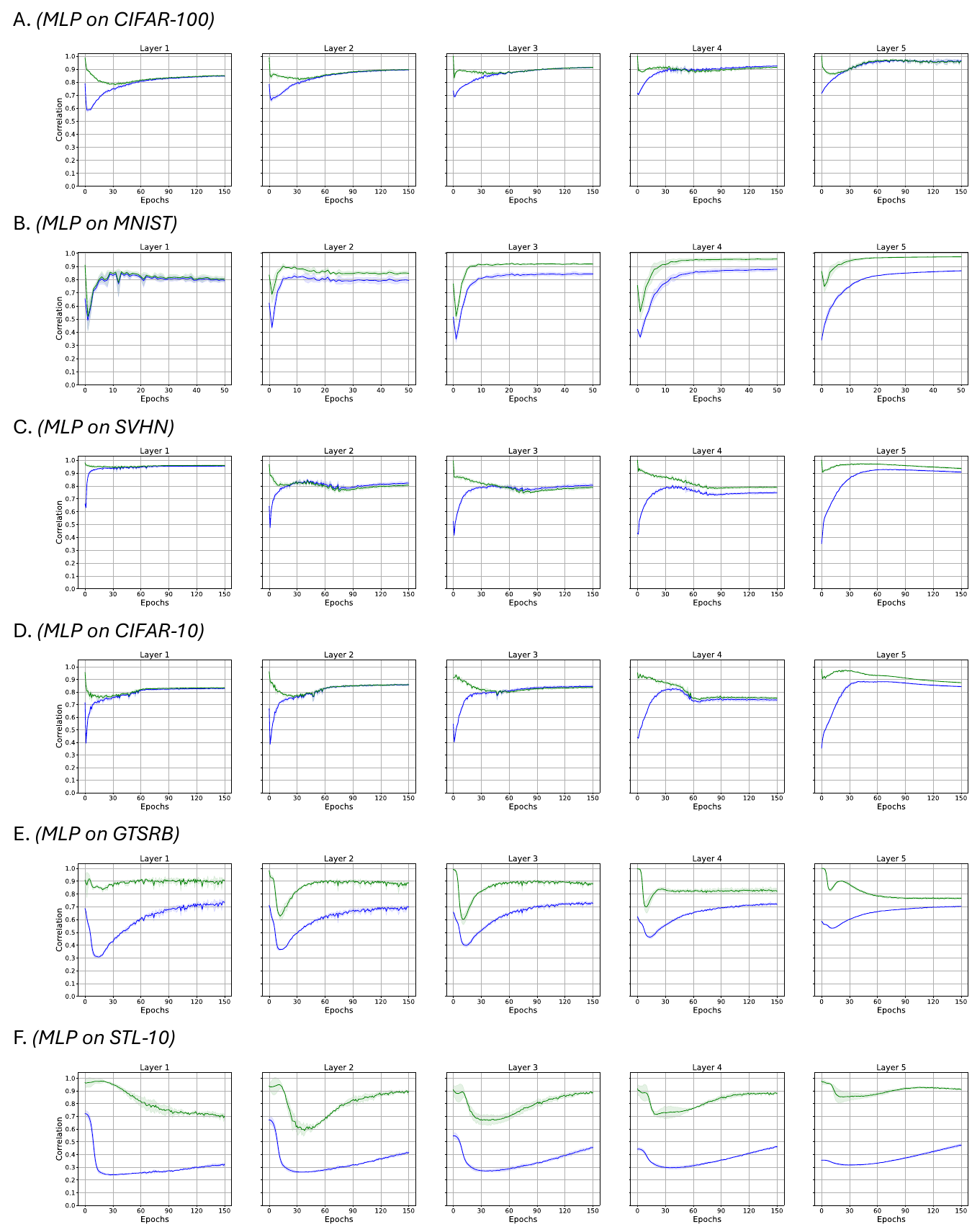}
    \caption{Full trajectories for the C/UC-NFC of a five hidden layer MLP trained on six datasets, averaged over three seeds, with large ($1.0$) initialization scale. The blue curves are the UC-NFC, while green curves are the C-NFC.}
    \label{fig: full NFA curves, MLP}
\end{figure*}

\begin{figure*}[h]
    \centering
    \includegraphics[scale=1.0]{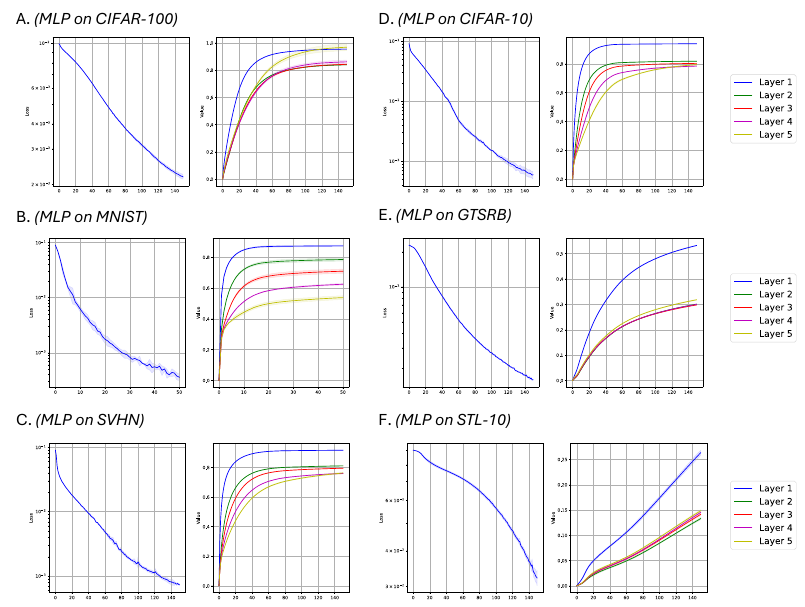}
    \caption{Losses and normalized changes in weights across datasets for a five hidden layer MLP. The change in weight is measured as $\|W - W_0\|\|W\|^{-1}$. First column of all subfigures are the losses, while the second columns are the weight changes.}
    \label{fig: losses, MLP}
\end{figure*}

\begin{figure*}[h]
    \centering
    \includegraphics[scale=0.625]{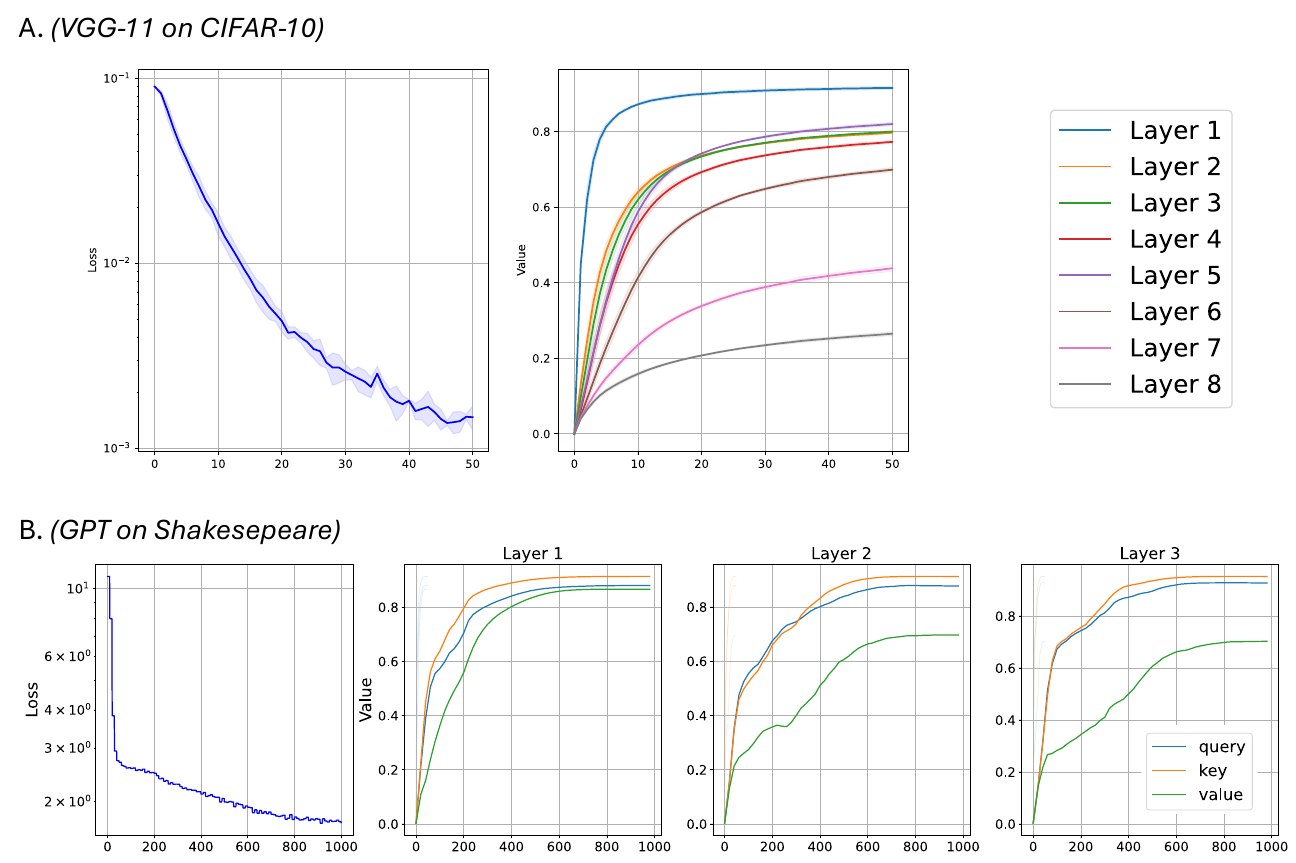}
    \caption{Losses and normalized changes in weights across datasets for VGG-11 on CIFAR-10 and GPT on Shakespeare character generation. The change in weight is measured as $\|W - W_0\|\|W\|^{-1}$.} First column of (A) and (B) are the losses, while the remaining columns are the weight changes.
    \label{fig: losses, VGG}
\end{figure*}

\newpage
\section{Additional experiments on real datasets}
We replicate Figures~\ref{fig: Corrupted Ansatz} and \ref{fig: NFA measurements, speed fixing} on celebrity faces (CelebA) and Street View House Numbers (SVHN). We begin by showing that one can disrupt the NFC correspondence by replacing the PTK feature covariance with a random matrix of the same spectral decay. For this example, we measure the Pearson correlation, which subtracts the mean of the image. I.e. $\bar{\rho}(A,B) \equiv \corr{A - m(A), B - m(B)}$, where $m(A), m(B)$ are the average of the elements of $A$ and $B$. 

\label{app: real datasets}
\begin{figure*}[h]
    \centering
    \includegraphics[scale=0.65]{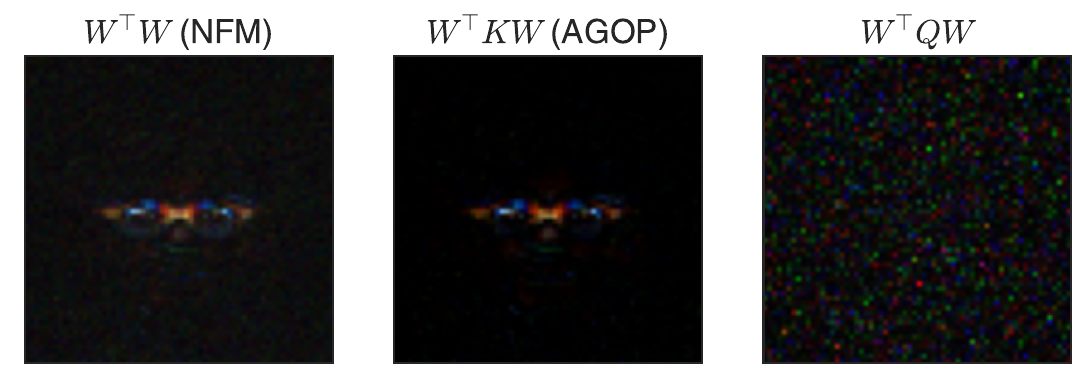}
    \caption{Various feature learning measures for the CelebA binary subtask of predicting glasses. The diagonals of the NFM
    $\round{W\tran W}$ (first plot) and AGOP $\round{W\tran K W}$
    (second plot) of a fully-connected network are similar to each other. Replacing $K$ with a symmetric matrix $Q$ with
    the same spectrum but independent eigenvectors obscures the
    low rank structure (third plot), and reduces the Pearson correlation of the diagonal from
    $\Pcorr{\diag{\NFM}, \diag{\AGOP}} = 0.91$ to $\Pcorr{\diag{\NFM}, \diag{W\tran Q W}} = 0.04$.}
  \label{fig: Corrupted Ansatz, CelebA}
\end{figure*}

\begin{figure*}[h]
    \centering
    \subfloat[][Feature learning measurements.]{\includegraphics[scale=0.5]{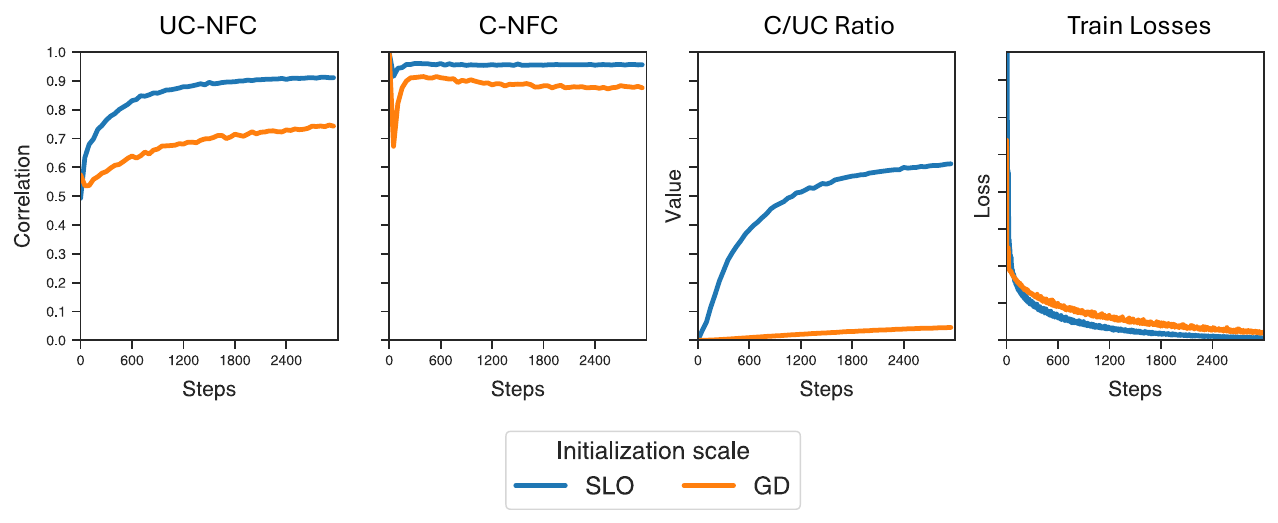}}\\
    \subfloat[][Standard Training]{\includegraphics[scale=0.55]{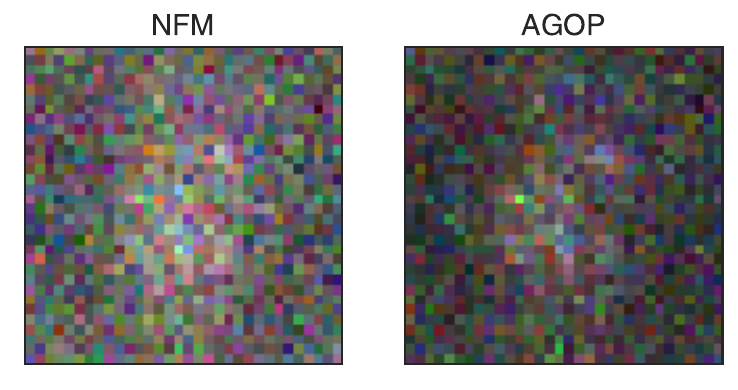}}
    \qquad 
    \subfloat[][\OptNameAcronym{}]{\includegraphics[scale=0.55]{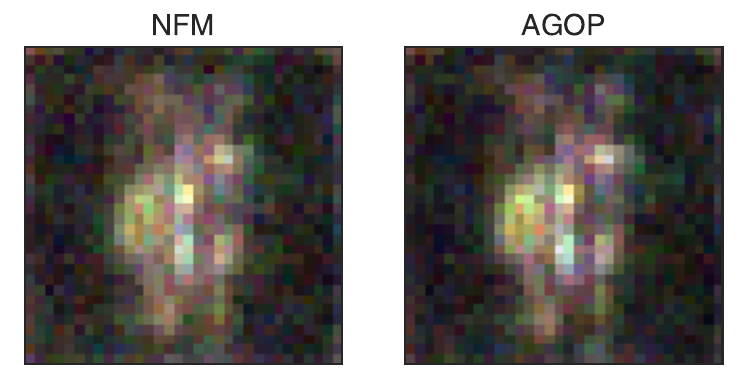}}
    \caption{We demonstrate on the SVHN dataset, with a 4 hidden layer neural network with large initialization scale, how \OptNameAcronym{} can improve the strength of the \UCNFC{}, the \CNFC{}, the ratio of the unnormalized \CNFC{} to \UCNFC{} (plot (a)) and the feature quality (plots (b) and (c)). In plots (b) and (c), we visualize the diagonal of the NFM and AGOP for the first layer of the trained network, where \OptNameAcronym{} was applied with $C_0=2.5$, $C_1=C_2=0.4$.}
  \label{fig: SLO experiments, SVHN}
\end{figure*}

\end{document}